\documentclass[a4paper,11pt]{article}

\usepackage[utf8]{inputenc}
\usepackage[english]{babel}
\usepackage{amsmath,amsthm,amssymb}
\usepackage{tikz}
\usepackage[numbers]{natbib}
\usepackage[left=2.05cm,right=2.05cm,top=2.95cm,bottom=2.8cm]{geometry}

\usepackage{algorithm}%
\usepackage{algorithmicx}%
\usepackage{algpseudocode}

\usepackage[parfill]{parskip}

\usepackage[pdfborder={0 0 0}]{hyperref}

\newcommand{\R}{\mathbb{R}}

\newcommand{\N}{\mathbb{N}}

\newcommand{\E}{\mathbb{E}}

\newcommand{\sX}{\mathcal{X}}

\newcommand{\KL}[2]{\mathrm{KL}({#1}\|{#2})}

\newcommand{\bx}{\mathbf{x}}

\newcommand{\bz}{\mathbf{z}}
\newcommand{\bs}{\mathbf{s}}

\newcommand{\ba}{\mathbf{a}}

\def \rmd {\mathrm{d}}

\newcommand{\pd}{\pi} 
\newcommand{\pth}{p^{\theta}} 
\newcommand{\nuth}{\nu^{\theta}} 
\newcommand{\pb}{p_{\textrm{alg}}} 
\newcommand{\Unif}{\mathrm{Unif}}
\newcommand{\Geom}{\mathrm{Geom}}

\newcommand{\bigO}{\mathcal{O}}
\newcommand{\smallo}{o}

\ifdefined\nonewproofenvironments\else
\ifdefined\ispres\else
\newtheorem{theorem}{Theorem}
\newtheorem{lemma}[theorem]{Lemma}
\newtheorem{proposition}[theorem]{Proposition}

\theoremstyle{definition}

\newtheorem{remark}[theorem]{Remark}

\newtheorem{assumption}{Assumption}

\makeatletter
\renewenvironment{proof}[1][\proofname] {\par\pushQED{\qed}\normalfont\topsep6\p@\@plus6\p@\relax\trivlist\item[\hskip\labelsep\bfseries#1\@addpunct{.}]\ignorespaces}{\popQED\endtrivlist\@endpefalse}
\makeatother

\begin{document}

\title{Error Bounds and Optimal Schedules for \\ Masked Diffusions with Factorized Approximations}
\author{Hugo Lavenant\thanks{Bocconi University, Department of Decision Sciences and BIDSA, Milan, Italy (hugo.lavenant@unibocconi.it)} \, and Giacomo Zanella\thanks{Bocconi University, Department of Decision Sciences and BIDSA, Milan, Italy (giacomo.zanella@unibocconi.it)}}
\date{\today}

\maketitle

\begin{abstract}
Recently proposed generative models for discrete data, such as Masked Diffusion Models (MDMs), exploit conditional independence approximations to reduce the computational cost of popular Auto-Regressive Models (ARMs), at the price of some bias in the sampling distribution. 
We study the resulting computation-vs-accuracy trade-off, 
providing general error bounds (in relative entropy) that depend only on the average number of tokens generated per iteration and are independent of the data dimensionality (i.e.\ sequence length), thus supporting the empirical success of MDMs. We then investigate the gain obtained by using non-constant schedule sizes (i.e.\ varying the number of unmasked tokens during the generation process) and identify the optimal schedule as a function of a so-called information profile of the data distribution, thus allowing for a principled optimization of schedule sizes.
We define methods directly as sampling algorithms and do not use classical derivations as time-reversed diffusion processes, leading us to simple and transparent proofs.
\end{abstract}

\section{Set-up, background and objectives}
Assume we are interested in generating samples from a probability distribution $\pd$ on a product space $\sX^N$, where $\sX$ can be a finite set of tokens or some other general state space.
A standard approach to generate a sample $\bx = (x_1, \ldots, x_N)$ is to proceed sequentially, sampling first $x_1$ from its marginal distribution $\pi(x_1)$ and then $x_i$ from its conditional distribution $\pi(x_i|\bx_{<i})$, for $i=2,\dots,N$, where $\bx_{<i}=(x_1,\dots,x_{i-1})$.
One limitation of this approach, underlying popular auto-regressive generative models (ARMs), is the need to perform $N$ sequential steps, which limits the speed and computational efficiency of the resulting algorithms.
This has motivated the exploration of alternative procedures to generate samples from $\pi$ given access to its conditional distributions (or approximations thereof), such as masked diffusion models (MDMs) \citep{austin2021structured,campbell2022continuous,shi2024simplified,sahoo2024simple}.
Despite being originally derived by analogy with diffusion models on continuous state spaces, MDMs can also be understood as a way to reduce the cost of standard ARMs by generating multiple tokens simultaneously, see e.g.\ \cite[Sec. 2.1.2]{kim2025train} and references therein for more discussion.
In this work, we follow this more algorithmic perspective, directly defining a general class of `unmasking' (or `sequential sampling') algorithms and analysing their properties and the resulting computational-vs-accuracy trade-off.

\paragraph{General unmasking algorithms} 
We consider sampling algorithms of the form described in Algorithm~\ref{alg:gen}. 
At iteration $k$, the algorithm decides which new components $z_k\subseteq \{1,\dots,N\}\backslash \bz_{<k}$ to generate, where $\bz_{<k}=\cup_{j=1}^{k-1}z_j\subseteq 
\{1, \ldots, N \}$ denotes the components of $\bx$ that have already been generated before iteration $k$, then samples $\bx_{z_{k}}=(x_i)_{i\in z_{k}}$ from some probability distribution $\pth(\bx_{z_{k}}; \bx_{\bz_{<k}})$ over $\sX^{|z_k|}$, and finally updates $\bz_{\leq k} = z_k \cup \bz_{< k}$.
Here $\pth(\bx_{z_{k}}; \bx_{\bz_{<k}})$ is some parametric approximation of $\pi(\bx_{z_{k}}|\bx_{\bz_{<k}})$, 
the conditional distribution of $\bx_{z_{k}}$ given $\bx_{\bz_{<k}}=(x_i)_{i\in\bz_{<k}}$ under $\pi$. 
The set $z_{k}$ is sampled from some probability distribution  $\nu^\theta(z_{k} ; \bz_{<k}, \bx_{\bz_{< k}})$ over subsets of $ \{1,\dots,N\}\backslash \bz_{<k}$, which can in principle depend on both $\bz_{<k}$ and $\bx_{\bz_{< k}}$. 
The algorithm continues until $\bz_{\leq k}$ coincides with the whole set $\{ 1, \ldots, N \}$ and the number of iterations required to terminate is denoted as $K=\inf\{k\,:\,\bz_{\leq k}=\{ 1, \ldots, N \}\}$, which is in general a random quantity. 
One often assumes $|z_k|\geq 1$ for all $k$, so that $K\leq N$. The ARM case corresponds to $z_k = \{ k \}$, with $K = N$ and $\bz_{\leq k} = \{1, \ldots, k \}$. 
Motivated by MDMs, components already sampled are referred to as `unmasked', whereas those yet to be sampled are the `masked' ones.

\begin{algorithm}[H]
\caption{Sequential sampling/unmasking}
\label{alg:gen}
\begin{algorithmic}
\Repeat{~for $k=1,2,\ldots$}
\State Sample $z_k\sim \nuth(z_k;\bz_{<k},\bx_{\bz_{<k}})$ subset of $\{1,\dots,N\}\backslash \bz_{<k}$ \hfill \emph{(Choose coordinates to unmask)}
\State Sample $\bx_{ z_k}\sim \pth(\bx_{ z_k};\bx_{\bz_{<k}})$ in $\sX^{|z_k|}$ \hfill \emph{(Generate tokens)} 
\Until{$\bz_{\leq k} = \{ 1, \ldots, N \}$.}
\State Set $K=\inf\{k\,:\,\bz_{\leq k} = \{ 1, \ldots, N \}\}$.
\State \Return $\bx = (x_1, \ldots, x_N) \in \sX^N$ and $(z_1, \ldots, z_K)$ an ordered partition of $\{ 1, \ldots, N \}$.
\end{algorithmic}
\end{algorithm}


Upon termination, the algorithm produces a sample $\bx$ in $\sX^N$ and an ordered partition $\bz=(z_1, \ldots, z_K)$ of $\{ 1, \ldots, N \}$.
By construction, their joint distribution reads
\begin{equation}\label{eq:dec_pb_prod}
\pb(\bx,\bz) = \prod_{k=1}^K \nuth(z_k;\bz_{<k},\bx_{\bz_{<k}}) \pth(\bx_{z_k};\bx_{\bz_{<k}})
=
\pth(\bx;\bz)\nuth(\bz;\bx),    
\end{equation}
where 
$\pth(\bx;\bz)  =\prod_{k=1}^K \pth(\bx_{z_k};\bx_{\bz_{<k}})$, 
$\nuth(\bz;\bx)=\prod_{k=1}^K \nuth(z_k;\bz_{<k},\bx_{\bz_{<k}})$ and for $k=1$ we use the notation $\bz_{<k}=\emptyset$ and $\pth(\bx_{z_k};\bx_{\bz_{<k}})=\pth(\bx_{z_1})$.

The aim of the algorithm is to produce high quality samples from $\pi$ in $K< N$ steps. 
This encompasses two objectives: the first is to make the distribution of the output $\bx$ of the algorithm as close as possible to $\pi$, that is to achieve 
$$
\pb(\bx)=\sum_{\bz}\pb(\bx,\bz)\approx \pd(\bx)\,.
$$
The second is to reduce the computational cost required to generate samples which, as detailed later, is proportional to $K$.
These two objectives are in competition with each other and result in a trade-off between sampling accuracy and computational cost.

\paragraph{Factorized approximations and sources of error}
If $K<N$ then multiple tokens need to be generated simultaneously, said differently the size of $z_k$ can be strictly greater than $1$. Since $|\sX|$ is potentially large, learning multivariate distributions over $|\sX|^s$ with $s>1$ is often unfeasible. 
Thus, one typically resorts to factorized approximations defined as
\begin{align}\label{eq:fact_ass}
\pth(\bx_{z_k};\bx_{\bz_{<k}})=\prod_{i\in z_k}\pth(x_i;\bx_{\bz_{<k}}).
\end{align}
Equivalently, tokens in $\bx_{z_k}$ are sampled as if they were independent conditionally to $\bx_{\bz_{<k}}$.
This way, at the price of an additional approximation, one only needs to learn univariate distributions given an arbitrary conditioning set.

We thus have two sources of error that make $\pb(\bx)$ different from $\pd(\bx)$:
\begin{enumerate}
\item (Learning conditionals)
We do not know the true conditionals of $\pd$ but rather learn them, leading to the approximation
$$
\pd(x_i|\bx_{\bz_{<k}})
\approx 
\pth(x_i;\bx_{\bz_{<k}})\,.
$$
\item (Factorized assumption) We generate multiple tokens simultaneously pretending they were independent, leading to the approximation
$$
\pd(\bx_{z_k}|\bx_{\bz_{<k}})
\approx \prod_{i\in z_k}
\pd(x_i|\bx_{\bz_{<k}})\,.
$$
\end{enumerate}
The first error is the classical one, also incurred by standard ARMs, related to learning conditional distributions of $\pd$ from samples (i.e.\ from a training data set).
The second one relates to the computation-vs-accuracy trade-off in sampling from a string of $N$ tokens in $K<N$ rounds. 
This intuition is formalized below in Proposition~\ref{prop:error}, where the sampling error in relative entropy is explicitly decomposed into a learning error $E_\text{learn}$ and a factorization error $E_\text{fact}$.

\paragraph{Computational cost and arbitrary planners}

In the generative models literature, $\nuth$  and $\pth$ are referred to as, respectively, \emph{planner} and \emph{denoiser}, see e.g. \citep{peng2025path,kim2025train} and references therein.

The denoiser $\pth$ is usually trained by minimizing a cross-entropy loss of the form
\begin{equation}\label{eq:loss}
\sum_{(i,z)} \omega(i,z) \E_{\pd(\bx)}\left[ 
\log\frac{\pd(x_i|\bx_{z})}{ \pth(x_i;\bx_{z})}
\right], 
\end{equation}
where the sum runs over $z$ subset of $\{ 1, \ldots, N \}$ and $i \notin z$, with weights $\omega(i,z)$ which typically are chosen to be uniform; see e.g.\ \cite[Sec. 3]{uria2014deep} and \cite{shi2024simplified}. By construction, this loss is minimized by the exact conditionals of $\pi$. 
Models of $\pth$ used in practice receive $\bx_{\bz_{<k}}$ as input and produce all univariate conditional distributions $\{\pth(x_i;\bx_{\bz_{<k}})\}_{i\notin \bz_{<k},x_i\in\sX}$ as output. 
Once the model $\pth$ has been evaluated, the cost of sampling from the $|z_k|$ univariate conditionals in \eqref{eq:fact_ass} is comparatively small.
As a result, the dominant cost in Algorithm \ref{alg:gen} is given by the $K$ evaluations of the model $\pth$, which is why we measure the computational cost of Algorithm \ref{alg:gen} with $K$.
See e.g.\ the discussion about sampling efficiency of MDMs in \citet{ben2025accelerated}, where $K$ is referred to as \emph{number of function evaluations}.

The planner $\nuth$ is a design choice that can be freely optimized to improve accuracy, i.e.\ make $\pb(\bx)$ as close as possible to $\pd(\bx)$.
It is worth noting that in principle $\nuth$ can be any selection rule, where $z_k$ can depend on both $\bz_{<k}$ and $\bx_{\bz_{<k}}$, without affecting the validity of the sampling algorithm, in the sense that with a perfect denoiser any planner would produce perfect samples. The proof of the following elementary result can be found in the appendix.

\begin{lemma}\label{lemma:arbitrary_schedule}
If $\pth(\bx_z;\bx_{z'})=\pd(\bx_{z}|\bx_{z'})$ for any $z, z'$ disjoint subsets of $\{ 1, \ldots, N \}$, then $\pb(\bx)=\pd(\bx)$, regardless of the choice of $\nuth$.
\end{lemma}

\paragraph{Objective and Contributions} 
\emph{Our main goal is to analyze the factorization error $E_\text{fact}$ incurred in Algorithm~\ref{alg:gen} by the factorized assumption~\eqref{eq:fact_ass}, answering the questions: how does it scale with $N$ and $K$? And how can one choose the planner $\nuth$ to minimize it?}

We assume that the training of the denoiser has already taken place, thus treat $\pth$ as given and fixed, and focus on optimizing the planner $\nuth$. 
We concentrate in particular on the case where the schedule is chosen with a random order (see Section \ref{sec:random_order} for a precise definition), where we obtain an upper bound (Theorem~\ref{thm:E_fac_exch_bound}) which is a factor $K$ better than the worst case bound over all schedules and all distributions $\pd$ (Proposition~\ref{prop:E_fac_bound}). Our analysis leads to an elegant rewriting of the factorization error in terms of the information profile of $\pd$ (Lemma~\ref{lm:value_Efact_random}). 
For a given information profile, we look at the problem of finding the schedule sizes that minimize the factorization error: by a scaling limit we connect it to a classical problem of calculus of variations (Theorems~\ref{thm:scal_lim_diverge} and~\ref{thm:scal_lim_bounded}). This results opens the way for a data-driven selection of an optimal schedule (Equation~\ref{eq:data_driven_schedule}).

\paragraph{Concurrent work} Similar results to ours were derived independently in the recent paper \cite{chen2025optimal}, which appeared online shortly after the first version of this work. In particular, \cite{chen2025optimal} also relates the factorization error to the information profile of $\pi$ (see Result 1 and Theorem 1.4 therein) the same way our Lemma \ref{lm:value_Efact_random} does.

\section{Error decomposition}

\paragraph{Decomposition of KL error} 
The approximation error between $\pb$ and the target distribution $\pd$ is measured with the Kullback-Leibler divergence, defined as
\begin{equation*}
\KL{\pd(\bx)}{\pb(\bx)}
    =\E_{\pi(\bx)}\left[\log\left(\frac{\pi(\bx)}{\pb(\bx)}\right)\right].
\end{equation*}
We will use the monotonicity of the Kullback-Leibler divergence: the $\mathrm{KL}$ decreases if we marginalize, and it also decreases (in expected value) if we condition. Both properties come from the chain rule for $\mathrm{KL}$ \cite[Theorem 2.5.3]{CoverThomas2006}.

The following proposition decomposes the sampling error in terms of the error due to the approximation in learning the conditionals, which we denote as $E_\text{learn}$, and the one due to the factorized assumption, which we denote as $E_\text{fact}$.
Below we denote the conditional total correlation of $(x_i)_{i\in z_k}$ given $\bx_{\bz_{<k}}$ under $\pd$ as
\begin{align*}
\mathrm{TC}_{\pd}(z_k|\bx_{\bz_{<k}})
=
\KL{\pd(\bx_{z_k}|\bx_{\bz_{<k}})}{\otimes_{i\in z_k}
\pd(x_i|\bx_{\bz_{<k}})}
=
\E_{\pd(\bx_{z_k}|\bx_{\bz_{<k}})}\left[
\log \frac{\pd(\bx_{z_k}|\bx_{\bz_{<k}})}{\prod_{i\in z_k}
\pd(x_i|\bx_{\bz_{<k}})}
\right]\,,\label{eq:mutual_info}
\end{align*}
where $\otimes$ denotes the independent product of distributions.
It measures how correlated the components of $x_i$, $i \in z_k$ are, given $x_j$, $j \in \bz_{<k}$, see e.g. \cite[Sec. 4]{Austin2020Multivariate} and references therein. By construction, $\mathrm{TC}_\pd(z_k|\bx_{\bz_{<k}})$ is non-negative and it equals $0$ if and only if the $x_i$, $i \in z_k$ are independent given $\bx_{\bz_{< k}}$, in particular this is always the case if $z_k$ contains exactly one element.

Note that $\pd(\bx) \nuth(\bz;\bx)$ is a valid probability mass function 
and that, if $(\bx,\bz)$ is drawn according to it, then $\bx \sim \pd$ and $\bz|\bx\sim\nuth(\bz;\bx)$.

\begin{proposition}
\label{prop:error}
Let $\pb(\bx)$ be the distribution induced by Algorithm~\ref{alg:gen} with $\pth$ as in \eqref{eq:fact_ass}. Then
\begin{equation*}
    \KL{\pd(\bx)}{\pb(\bx)}
\leq
E_\text{learn}+E_\text{fact}
\end{equation*}
where
\begin{equation*}
E_\text{learn}=\E_{\pd(\bx)\nuth(\bz;\bx)}\left[
\sum_{k\geq 1}
\sum_{i\in z_k}
\log\frac{\pd(x_i|\bx_{\bz_{<k}})}{\pth(x_i;\bx_{\bz_{<k}})}
\right]
\,, \qquad
E_\text{fact}=\E_{\pd(\bx)\nuth(\bz;\bx)}\left[
\sum_{k\geq 1}
\mathrm{TC}_\pd(z_k|\bx_{\bz_{<k}})
\right]\,.
\end{equation*}
\end{proposition}

\begin{proof}
We use the monotonicity of $\mathrm{KL}$:
as $\pd(\bx) \nuth(\bz;\bx)$ has marginal distribution $\pd(\bx)$, 
\begin{equation*}
\KL{\pd(\bx)}{\pb(\bx)} \leq \KL{\pd(\bx) \nuth(\bz;\bx)}{\pb(\bx,\bz)} = \KL{\pd(\bx)\nuth(\bz;\bx)}{\pth(\bx;\bz) \nuth(\bz;\bx)}, 
\end{equation*}
where the second equality comes from~\eqref{eq:dec_pb_prod}. We expand the definition of the KL divergence:
\begin{equation*}
\KL{\pd(\bx)\nuth(\bz;\bx)}{\pth(\bx;\bz) \nuth(\bz;\bx)} = \E_{\pd(\bx)\nuth(\bz;\bx)}\left[\log \frac{\pd(\bx)}{\pth(\bx;\bz)}\right].
\end{equation*}
Following~\eqref{eq:fact_ass}, we can write 
$$
\log \frac{\pd(\bx)}{\pth(\bx;\bz)}
=
\sum_{k\geq 1}\left(
\log \frac{\pd(\bx_{z_k}|\bx_{\bz_{<k}})}{\prod_{i\in z_k}
\pd(x_i|\bx_{\bz_{<k}})}
+
\log\frac{\prod_{i\in z_k}\pd(x_i|\bx_{\bz_{<k}})}{\prod_{i\in z_k}\pth(x_i;\bx_{\bz_{<k}})}\right)
\,.
$$
Taking the expectation with respect to $\pd(\bx)\nuth(\bz;\bx)$ gives $\KL{\pd(\bx)\nuth(\bz;\bx)}{\pth(\bx;\bz) \nuth(\bz;\bx)}=E_\text{fact}+E_\text{learn}$ since the conditional distribution of $\bx_{z_k}$ given $\bz_{\leq k}$ and  $\bx_{\bz_{<k}}$ under $\pd(\bx)\nuth(\bz;\bx)$ coincides with $\pd(\bx_{z_k}|\bx_{\bz_{<k}})$.
\end{proof}

The term $E_\text{learn}$ measures the closeness of $\pth$ to $\pd$ and it is zero if $\pth(x_i;\bx_{z})=\pd(x_i|\bx_{z})$ for all $z$ subset of $\{ 1, \ldots, N \}$ and $i \notin z$. It is very close to the loss function minimized during training recalled in~\eqref{eq:loss}: the only difference between $E_\text{learn}$ and the learning loss in~\eqref{eq:loss} is the weights $\omega$.

On the contrary, the factorization error $E_\text{fact}$ is independent of $\pth$ and has a strong dependence on $K$. In particular, it is zero when $K=N$ (meaning $|z_k| = 1$ for all $k$) and it generally increases as $K$ decreases: this is consistent as it is related to the factorization approximation.
In the sequel we focus only on the factorization error, $E_\text{fact}$.

Error bounds involving conditional mutual informations already appeared in the literature on discrete diffusion models \citep{park2024optimizing}. In particular the recent works \citet{li2025convergence} and \citet{ben2025accelerated} exploit decompositions analogous to the one in Proposition \ref{prop:error}.

\section{Worst-case bounds on the factorization error}

We first concentrate on an arbitrary planner $\nuth$ and we explain how $E_\text{fact}$ may scale in this case. Inspired by the bounds in \citet{li2025convergence}, we consider the following measure of correlation for the distribution $\pd$:
\begin{equation*}
D(\pd)=\frac{1}{N}\sum_{i=1}^N
\E_{\pd(\bx)} \left[ \log \frac{\pd(\bx)}{\pd(x_i)\pd(\bx_{-i})} \right]
= \frac{1}{N} \sum_{i=1}^N \E_{\pd(\bx)} \left[ \log \frac{\pd(x_i|\bx_{-i})}{\pd(x_i)} \right] \,,
\end{equation*}
where $\bx_{-i}=(x_j)_{j\neq i}$. 
As noted in \cite[Lemma 4.3]{Austin2020Multivariate}, $N D(\pd)$ coincides with sum of the total correlation and dual total correlation of $\pi$.
As $\log \pd(x_i|\bx_{-i}) \leq 0$ and with the concavity of $\log$ it is straightforward to see that $D(\pd) \leq \log |\sX|$. 
We prove two upper bounds: one universal not depending on $\pd$, and one a bit finer depending both on $D(\pd)$ and the schedule. 

\begin{proposition}
\label{prop:E_fac_bound}
We have
\begin{align}\label{eq:E_fac_bound_gen}
E_\text{fact}
\leq 
(N-\E_{\pd(\bx)\nuth(\bz;\bx)}[K])\log |\sX|.
\end{align}
Moreover, if $\nuth(\bz)=\nuth(\bz;\bx)$ is independent of $\bx$, denoting $s_k=|z_k|$ and $s_{\max}=\max (s_1, \ldots, s_K)$, we have
\begin{align}\label{eq:E_fac_bound_gen_2}
E_\text{fact}\leq \left(N-\E_{\nuth(\bz)}\left[\frac{N}{s_{\max}}\right]\right) D(\pd)\,.
\end{align}
\end{proposition}

\begin{proof}
To prove \eqref{eq:E_fac_bound_gen},
we use the following bound for the total correlation: for any ordering $i_1, \ldots i_{s_k}$ of $z_k$, as $\pd(\bx_{z_k}| \bx_{\bz_{<k}}) = \prod_{\ell=1}^{s_k} \pd(x_{i_\ell}|\bx_{\bz_{<k} \cup \{ i_1, \ldots i_{\ell -1} \}})$ we have 
\begin{align}
\mathrm{TC}_\pd(z_k|\bx_{\bz_{< k}}) 
&=
\sum_{\ell = 2}^{s_k} \E_{\pd(\bx_{z_k}|\bx_{\bz_{< k}})}\left[
\log 
\frac{\pd(x_{i_\ell}|\bx_{\bz_{<k} \cup \{ i_1, \ldots i_{\ell -1} \}})}{
\pd(x_{i_\ell}|\bx_{\bz_{<k}})} \right]
\label{eq:TC_bound_1}
\\
&
\leq 
-\sum_{\ell = 2}^{s_k} \E_{\pd(\bx_{z_k}|\bx_{\bz_{< k}})}\left[
\log \pd(x_{i_\ell}|\bx_{\bz_{<k}}) \right]
\leq
(s_k-1)\log |\sX|\,.\nonumber
\end{align}
Summing over $k$ and using $\sum_k s_k = N$ gives
$\sum_{k\geq 1}\mathrm{TC}_\pd(z_k|\bx_{\bz_{< k}})\leq (N-K)\log|\sX|$, and then \eqref{eq:E_fac_bound_gen} follows from the definition of $E_\text{fact}$. 

We now prove \eqref{eq:E_fac_bound_gen_2}.
If $\nuth(\bz)=\nuth(\bz;\bx)$ then 
$E_\text{fact}=\E_{\nuth(\bz)}\left[
\sum_{k\geq 1}
\E_{\pd(\bx_{z_{<k}})}[
\mathrm{TC}_\pd(z_k|\bx_{\bz_{<k}})
]
\right]$. 
Taking expectations with respect to $\bx_{z_{<k}}$ in \eqref{eq:TC_bound_1}, we obtain
\begin{equation*}
\E_{\pi(\bx_{z_{<k}})}[\mathrm{TC}_\pd(z_k|\bx_{\bz_{< k}})] = \sum_{\ell = 2}^{s_k} \E_{\pd(\bx)}\left[
\log \frac{\pd(x_{i_\ell}|\bx_{\bz_{<k} \cup \{ i_1, \ldots i_{\ell -1} \}})}{
\pd(x_{i_\ell}|\bx_{\bz_{<k}})} \right] 
\leq \sum_{\ell = 2}^{s_k} \E_{\pd(\bx)}\left[
\log \frac{\pd(x_{i_\ell}| \bx_{- i_\ell})}{
\pd(x_{i_\ell})} \right],
\end{equation*}where the inequality follows by the monotonicity of $\mathrm{KL}$ with respect to conditioning. Averaging in the right hand side over all possible ordering of $z_k$, we obtain
\begin{equation*}
\E_{\pi(\bx_{z_{<k}})}[\mathrm{TC}_\pd(z_k|\bx_{\bz_{< k}})]
\leq
\left( 1 - \frac{1}{s_k} \right) \sum_{i \in z_k} \E_{\pd(\bx)}\left[
\log \frac{\pd(x_i| \bx_{- i})}{
\pd(x_i)} \right]. 
\end{equation*}
From the above, and using $\sum_k \sum_{i\in \bz_k}
\E_{\pd(\bx)} \left[ \log \frac{\pd(x_i|\bx_{-i})}{\pd(x_i)} \right]= ND(\pd)$ for any partition $\bz$, as well as $s_k\leq s_{\max}$,
we obtain \eqref{eq:E_fac_bound_gen_2}.
\end{proof}

Both bounds in Proposition \ref{prop:E_fac_bound} can be saturated by adversarial choices of $\pi$ and $\bz$.

\begin{lemma}
\label{lem:example}
The inequality in \eqref{eq:E_fac_bound_gen} is an equality in the following case: if $\bz$ is deterministic and under $\pd$, the vector $\bx$ has uniform marginals laws, with $x_i = x_j$ a.s. for couples $(i,j)$ in the same set in the partition $\bz$ while $\bx_{z_1}, \ldots, \bx_{z_K}$ are all independent;
the inequality in \eqref{eq:E_fac_bound_gen_2} 
is an equality if in addition $s_k=N/K$ for all $k$.
\end{lemma}

\begin{proof}
A direct computation yields $\mathrm{TC}_\pd(z_k|\bx_{\bz_{< k}})  = (s_k-1) \log |\sX|$ and $D(\pd) = \log |\sX|$. 
Summing over $k$ gives $E_\text{fact}=(N-K)\log |\sX|$, i.e.\ equality in~\eqref{eq:E_fac_bound_gen} and equality in~\eqref{eq:E_fac_bound_gen_2} 
if $N/s_{\max} = K$. 
\end{proof}

If $s_k$ is roughly constant over $k$, then $s_{\max}\approx N/K$, so that the upper bound in \eqref{eq:E_fac_bound_gen_2} is approximately $(N-K)D(\pi)$. In particular if $K,N\to +\infty$ and $N/K\to \bar{s}> 1$, the bound in \eqref{eq:E_fac_bound_gen} grows linearly with $N$, resulting in very weak guarantees on the overall sampling error.
However, such worst-case bounds can be very pessimistic: below we show that if $\bz$ is appropriately randomized, then  $E_\text{fact}$ is provably of much smaller size.

\section{The random-order case}
\label{sec:random_order}

Consider now the situation where $z_k$ is sampled by first generating its size $s_k = |z_k|$ and then sampling its entries uniformly without replacement from $\{1,\dots,N\}\backslash \bz_{<k}$. This is the case in common implementations of MDMs \citep{shi2024simplified}. Specifically, to generate the ordered partition $\bz$, the algorithm first sample the sizes $(s_1, s_2, \ldots, s_K) = \bs$ with $\sum_{k=1}^K s_k = N$. Then, once these sizes are sampled, recursively
\begin{equation}
\label{eq:random_order}
z_k \; \text{  given  } \; (\mathbf{z}_{< k}, \bs) \; \text{  is a subset of } \; \{ 1, \ldots,N \} \setminus \bz_{<k} \; \text{ of size } \; s_k \; \text{ chosen uniformly at random.}
\end{equation}
Under \eqref{eq:random_order}, the distribution of $\bz$ is fully specified by the distribution of its sizes $\bs=(s_1,\dots,s_K)$, or equivalently their cumulative sums $\ba=(a_0,\dots,a_K)$ defined as $a_k = |\bz_{\leq k}|=\sum_{i=1}^ks_i$ with $a_0=0$. 
Thus, to define a planner $\nuth$ we only need to specify the law of $\ba$. 
With a slight abuse of notation, we denote by $\nuth(\ba)$ the law of $\ba$, which here we assume to be independent of $\bx$ for simplicity, and also by $\nuth(\bs)$ and $\nuth(\bz)$ the resulting laws induced on $\bs$ and $\bz$ by \eqref{eq:random_order}.

\subsection{Rewriting the factorization error with the information profile}

We show that, in the random order case, the factorization error has a convenient explicit dependence on the one-dimensional function $f$ defined as
\begin{align}
f(i)&=
\E_{\pd(\bx),\sigma\sim \Unif}\left[
\log \pd(x_{\sigma_{i+1}}|\bx_{\sigma_{\leq i}})
\right]
&i\in\{0,\dots,N-1\}\,.\label{eq:f_exc}
\end{align}
for $\sigma$ being a random permutation of $\{1,\dots,N\}$ uniformly distributed and $\bx_{\sigma_{\leq i}}=(x_{\sigma_j})_{j=1,\dots,i}$.
We refer to $f$ as `\emph{information
profile}' of $\pd$, see also \citet{bauer2008average} for similar terminology used in a different but analogous context. 

\begin{lemma}
\label{lm:prop_f}
For any $\pd$, the information profile $f$ is increasing and satisfies $f(N-1)-f(0)=D(\pd)$.
\end{lemma}

\begin{proof} 
By Jensen's inequality we have $\E_{\pd(\bx)}[\log \pd(x_{\sigma_{i+1}}| \bx_{\sigma_{< i}})] \leq \E_{\pd(\bx)}[\log \pd(x_{\sigma_{i+1}}| \bx_{\sigma_{\leq i}})]$. Averaging over $\sigma$ we find $f(i-1) \leq f(i)$.
Then we decompose $D(\pd)$ as
\begin{equation*}
D(\pd) = \frac{1}{N} \sum_{i=1}^N \E_{\pd(\bx)} \left[\log \left( \frac{\pd(x_i|\bx_{-i})}{\pd(x_i)} \right) \right] = \frac{1}{N} \sum_{i=1}^N \E_{\pd(\bx)}[ \log \pd(x_i|\bx_{-i}) ] - \frac{1}{N} \sum_{i=1}^N \E_{\pd(\bx)}[\log \pd(x_i)].
\end{equation*}
We recognize $f(N-1)$ in the first sum and $f(0)$ in the second one.
\end{proof}

The information profile is all we need to know about the distribution $\pd$ in order to evaluate the factorization error, as the next lemma shows.

\begin{lemma}
\label{lm:value_Efact_random}
Under \eqref{eq:random_order} we have
\begin{align}\label{eq:exch_version}
E_\text{fact}&=\E_{\nuth(\ba)}[A(\ba )]
&\hbox{ with }A(\ba )
= 
\sum_{i=0}^{N-1} f(i) -  \sum_{k=0}^{K-1}(a_{k+1} - a_k) f(a_k)\,.
\end{align}
\end{lemma}

Thus, under \eqref{eq:random_order}, the factorization error coincides with the error in the Riemann approximation
to the integral $\sum_{i=0}^{N-1}f(i)$ of the information profile $f$ with $K$ intermediate steps instead of $N$.

\begin{proof}
We use the definition of $E_\text{fact}$ and expand the total correlation to have
\begin{align}
E_\text{fact}&=
\E_{\pd(\bx)}[\log \pd(\bx)]-\E_{\pd(\bx)\nuth(\bz)}\left[
\sum_{k\geq 1}
\sum_{i\in z_k}
\log \pd(x_i|\bx_{\bz_{<k}})
\right]\,.
\end{align}
Since $\log \pd(\bx)=\sum_{i=0}^{N-1}\log \pd(x_{\sigma_{i+1}}|\bx_{\sigma_{\leq i}})$ for any permutation $\sigma$, it follows that 
\begin{equation*}
\E_{\pd(\bx)}[\log \pd(\bx)]=
\E_{\pd(\bx),\sigma\sim \Unif}\left[
\sum_{i=0}^{N-1}\log \pd(x_{\sigma_{i+1}}|\bx_{\sigma_{\leq i}})
\right]
=
\sum_{i=0}^{N-1} f(i). 
\end{equation*}
On the other hand, by~\eqref{eq:random_order}, if $\bz \sim \nuth$ and $i \in z_k$ uniformly at random then $(i,\bz_{<k})$ has distribution equal to $(\sigma_{a_{k-1} + 1}, \{\sigma_1, \ldots \sigma_{a_{k-1}} \})$ with $\sigma \sim \Unif$. Thus, since $|z_k| = s_k = a_k - a_{k-1}$, we obtain
\begin{align*}
\E_{\pd(\bx)\nuth(\bz)}\left[
\sum_{i\in z_k}
\log \pd(x_i|\bx_{\bz_{<k}})
\right]
& = \E_{\pd(\bx)\nuth(\bz), \sigma \sim \Unif}\left[
|z_k|
\log \pd(x_{\sigma_{a_{k-1} + 1}}|\bx_{\{\sigma_1, \ldots \sigma_{a_{k-1}} \}})
\right] \\ & = \E_{\nuth(\ba)}[(a_{k} - a_{k-1}) f(a_{k-1})].
\end{align*}
The conclusion follows by summing over $k$.
\end{proof}

\subsection{Resulting upper bounds}

We leverage the representation of the factorization error to obtain upper bounds in the random order case that scale much better than in the worst case.

\begin{theorem}
\label{thm:E_fac_exch_bound}
Under \eqref{eq:random_order} we have
$$E_\text{fact}
\leq
(\E_{\nuth(\bs)}[s_{\max}]-1)D(\pd)
\leq 
(\E_{\nuth(\bs)}[s_{\max}]-1)\log |\sX| \,.$$ 
\end{theorem}

The proof of Theorem \ref{thm:E_fac_exch_bound} relies on the following algebraic rewriting of the function $A$ defined in Lemma~\ref{lm:value_Efact_random}, whose proof can be found in the appendix.

\begin{lemma}
\label{lm:discrete_der}
Writing $\Delta f(i)=f(i)-f(i-1)$ for the discrete derivative of $f$,
for any $\ba $, we have
\begin{align}\label{eq:discr_der_bound}
&A(\ba )
=\sum_{i=1}^{N-1} \Delta f(i)(r_{\ba}(i)-i)
&
\text{with  } \; r_{\ba}(i)=\inf\{a_k\,:\,a_k\geq i\}\,.
\end{align}
\end{lemma}

\begin{proof}[Proof of Theorem \ref{thm:E_fac_exch_bound}]
For any $\ba$, as $a_{k+1} - a_k\leq s_{\max}$, we have $0 \leq r_{\ba}(i)-i\leq s_{\max} - 1$. Thus, as $\sum_{i=1}^{N-1} \Delta f(i)=D(\pd)$ and given Lemma~\ref{lm:value_Efact_random}, we have $A(\ba) \leq (s_{\max} - 1) D(\pd)$. The conclusion follows by~\eqref{eq:exch_version} and the bound $D(\pd) \leq \log |\sX|$.
\end{proof}

The bound in Theorem \ref{thm:E_fac_exch_bound} is minimized by taking schedules $\bz$ with near-constant sizes $(s_k)_{k}$, where one can enforce $s_{\max}\leq\lceil N/K\rceil$.
Here and below, we write $\lfloor x \rfloor$, $\lceil x \rceil$ for the largest (resp.\ smallest) integer smaller (resp.\ larger) than $x$.
This results in the bound $E_\text{fact}
\leq
\lceil (N-K)/K\rceil\log |\sX|$, which is a factor of $K$ better than the one in Proposition \ref{prop:E_fac_bound}, showing that random order schedules are guaranteed to perform drastically better than the worst case described in Proposition \ref{prop:E_fac_bound}. For example, if $N/K$ and $|\sX|$ are fixed, the bound in Theorem \ref{thm:E_fac_exch_bound} is, remarkably, independent of $N$. 
Note that in both in Proposition~\ref{prop:E_fac_bound} and Theorem~\ref{thm:E_fac_exch_bound} we recover that $E_\text{fact} = 0$ if we set $K=N$.

A bound similar to the one in Theorem \ref{thm:E_fac_exch_bound} was recently derived in \citet[Theorem 1]{li2025convergence} with a different and significantly less direct proof approach.

\subsection{Explicit computations with geometric schedules and lower bounds}

Interestingly, one can compute almost exactly the value of $E_\text{fact}$ for the case of random, geometrically distributed schedule sizes, without requiring any assumption on $\pd$.
Specifically, given $p\in(0,1)$ and $m\in\N$, let $\Geom (p;m)$ denote a Geometric distribution starting from $1$ and with a threshold at $m$, i.e.\ a random variable $X\sim \Geom (p;m)$ satisfies $\mathrm{Pr}(X=i)=(1-p)^{i-1}p$ for $i\in\{1,\dots,m-1\}$ and $\mathrm{Pr}(X=m)=(1-p)^{m-1}$. The proof of the following result relies on the memoryless property of the geometric distribution, it can be found in the appendix.

\begin{proposition}
\label{prop:E_fac_geom}
Assume we generate the sequence $\bs = (s_1, \ldots, s_K)$ as follows: $s_1\sim \Geom (p;N)$ and $s_k| \bs_{<k}\sim \Geom (p;N-\sum_{i=1}^{k-1}s_i)$ for $k=2,3\dots$. 
Then under~\eqref{eq:random_order} we have the upper bound
\begin{equation*}
E_\text{fact} \leq \frac{1-p}{p}D(\pd),
\end{equation*}
as well as the lower bound
\begin{equation*}
E_\text{fact} \geq \frac{1-p}{p}D(\pd) - \left( \frac{1-p}{p} \right)^2 \left( \max_{i=1,\dots,N-1} \Delta f(i) \right)\,.
\end{equation*}
\end{proposition}

If the information profile varies smoothly, given that $\sum_{i=1}^{N-1} \Delta f(i) = D(\pd)$, it is reasonable to expect $\max_i \Delta f(i) = \bigO(D(\pd)/N)$ as $N \to + \infty$, leading to the estimate 
\begin{equation*}
E_\text{fact} = \left( 1 + \bigO \left( \frac{1}{N} \right) \right) \cdot \frac{1-p}{p}D(\pd)
\end{equation*}
in the setting of Proposition~\ref{prop:E_fac_geom}. 
The $\bigO(1/N)$ terms relates to an `edge effect', that is, the truncation of the geometric random variables at $N$, see the proof of Proposition \ref{prop:E_fac_geom} in the appendix for more details. 
Here the number of steps $K$ is random, with $K\approx p N$ and $\E[s_k]\approx \frac{1}{p}$
or equivalently $\frac{1-p}{p} \approx (\E[s_k]-1)$ for all $k$. 
Thus, Proposition \ref{prop:E_fac_geom} can be interpreted as stating that
$$
E_\text{fact}
\approx
(\E[s_k]-1)D(\pd)\,.
$$
This suggests that the upper bound in Theorem \ref{thm:E_fac_exch_bound} is tight for schedules where $s_{\max}\approx \E[s_k]$. 
Actually by Markov's inequality, still in the limit $N \to + \infty$, it implies that $A(\ba )$ is of order $(\E[s_k]-1)D(\pd)$ with high probability when $\ba $ is sampled as in Proposition~\ref{prop:E_fac_geom}.

\subsection{The case of a fixed, randomly-generated ordering}

The above results apply to sampling strategies that randomize $\bz$ at every generation step, namely versions of 
Algorithm \ref{alg:gen} with $\nuth$ satisfying \eqref{eq:random_order}.
This, however, requires the user to have access to all conditionals, i.e.\ to learn  $\pth(x_i;\bx_{z})\approx\pd(x_i|\bx_{z})$ for every $z\subseteq \{1,\dots,N\}$ and $i\notin z$, since any combination of $i$ and $z$ could appear during sampling.

Nonetheless, analogous guarantees can be obtained for strategies that first pick a random order (independent of $\pd$), and then keep it fixed during generation. 
Indeed, since $E_\text{fact}=\E_{\nuth(\bz)}[E_\text{fact}(\bz)]$ 
with $E_\text{fact}(\bz) = \sum_{k\geq 1} \E_{\pi(\bx_{\bz_{< k}})}[\mathrm{TC}_\pd(z_k | \bx_{\bz_{< k}})]\geq 0$, a simple application of Markov's inequality yields $\Pr_{\nuth(\bz)}(E_\text{fact}(\bz)\geq c E_\text{fact})\leq 1/c$, so that choosing $s_{\max}\leq\lceil N/K\rceil$ Theorem~\ref{thm:E_fac_exch_bound} implies
\begin{equation}\label{eq:markov_fixed_z}
\mathrm{Pr}_{\nuth(\bz)}\left( E_\text{fact}(\bz)\geq c \left\lceil \frac{N-K}{K} \right\rceil D(\pd) \right)\leq \frac{1}{c}.
\end{equation}
Thus, one could instead first generate a schedule $\bz=(z_1,\dots,z_K)$, learn only a fixed set of conditionals, namely $\pth(x_i;\bx_{\bz_{<k}})\approx\pd(x_i|\bx_{\bz_{<k}})$ for every $k=1,\dots, K$ and $i\in z_{k}$, for such pre-specified $\bz$, and then apply Algorithm \ref{alg:gen} with such fixed $\bz$.
By \eqref{eq:markov_fixed_z}, this procedure would enjoy, with high probability, the same theoretical guarantees as Algorithm \ref{alg:gen} with $\nuth$ satisfying \eqref{eq:random_order} in terms of controlling $E_\text{fact}$, while potentially simplifying the process of learning $\pth$, see e.g.\ \citep{kim2025train}. 

We expect concentration results stronger than \eqref{eq:markov_fixed_z} to hold for $E_\text{fact}(\bz)$ as $N,K$ increase, possibly under some additional assumptions on the information profile of $\pd$, but leave those to future research.

\section{Optimal schedules and scaling limits}

We now consider the problem of minimizing $E_\text{fact}$ with respect to $\nuth$ for fixed $K$, in the random-order case defined in \eqref{eq:random_order}. To that end we recall that $E_\text{fact} = \E_{\nuth(\ba)}[A(\ba)]$ with $\ba$ encoding the size of $(\bz_{\leq k})_{k}$, see~\eqref{eq:exch_version}. Thus $E_\text{fact} \geq \min_{\ba} A(\ba)$, and we have equality if $\nuth(\ba)$ is deterministic and picks a minimizer of $A$. In other words, to minimize $E_\text{fact}$, it is enough to restrict to deterministic schedule sizes. 
This leads to the following optimization problem:
\begin{align}
\label{eq:optimal_schedule}
\min_{\ba =(a_0,\dots,a_K)}&
A(\ba )
&\hbox{given }
0=a_0<a_1<\dots<a_{K} = N\, ;
\end{align}
which is the focus of this section.

A question of interest is how much can be gained by using non-constant increments, i.e.\ deviating from the case $s_k = a_{k}-a_{k-1} \approx N/K$ for all $k$, where 
Theorem \ref{thm:E_fac_exch_bound} and Proposition \ref{prop:E_fac_geom} show that $E_\text{fact}$ is of order $(\E[s_{\max}]-1)D(\pd) \approx \frac{N-K}{K} D(\pd)$.
We first give a qualitative analysis in the case $N,K$ finite, and then we look at what happens when $N,K \to + \infty$. In the latter case the optimization problem~\eqref{eq:optimal_schedule} becomes a problem of calculus of variations, whose solution can sometimes be found explicitly. 

\subsection{A qualitative analysis of the optimal schedule}

We first note that the set of optimization variables in~\eqref{eq:optimal_schedule} is a finite set of cardinality ${N-1}\choose{K-1}$. Thus there always exists an optimal schedule, i.e. a solution to~\eqref{eq:optimal_schedule}, but finding it by enumeration is intractable as $N$ and $K$ grow.

A natural question is to know if optimal schedules should have increasing or decreasing sizes $s_k = a_{k} - a_{k-1}$. Intuitively, we should take $s_k$ increasing if most of the dependence structure of $\pd$ is captured by the first components that are sampled, that is, if conditionally to $\bx_{z}$ with $|z|$ small then the remaining components are weakly dependent. We show that this is connected to the convexity of the information profile.

Specifically we say that the information profile $f$ is (strictly) convex if $i \mapsto \Delta f(i)$ is (strictly) increasing. Analogously, $f$ is (strictly) concave if $i \mapsto \Delta f(i)$ is (strictly) decreasing. The proof of the following can be found in the appendix. 

\begin{proposition}
\label{prop:f_a_convex_concave}
Assume that $f$ is strictly increasing and $\ba$ solves~\eqref{eq:optimal_schedule}. Writing $s_k = a_k - a_{k-1}$, there holds for all $k\in\{0,\dots,K-1\}$,   
\begin{align}
\label{eq:sequential_sol}
s_{k+1} \in \left[\frac{f(a_{k}-1) - f(a_{k-1})}{\Delta f(a_k)},
\frac{f(a_{k}+1) - f(a_{k-1})}{\Delta f(a_k+1)}
\right].
\end{align}
In addition, if $f$ is strictly convex (resp.\ strictly concave) then $(s_k)_k$ is non-increasing (resp.\ non-decreasing).
\end{proposition}

The convexity versus concavity of the information profile depends on the dependance structure of $\pi$. For example, in Section \ref{sec:gauss_info_profile} of the appendix, we compute explicitly the information profile for exchangeable multivariate Gaussian distributions, where convexity versus concavity of $f$ is in one-to-one correspondence with  negative versus positive correlation among coordinates in $\pi$.

We will not analyse further the problem~\eqref{eq:optimal_schedule} without additional assumptions as it does not admit an analytical solution as far as we know. We only note that if the interval in \eqref{eq:sequential_sol} contains only one integer (or a few of them), then it gives a recursive relation to compute $a_{k+1}$ given $a_k, a_{k-1}$, which can be used to find an optimal schedule.

\subsection{Scaling limit: the setting}
\label{sec:scaling_limit_setting}

We turn to the limit $N, K \to + \infty$. 
We define the non-constant schedule $\ba$ through a continuous function:
given an increasing function $\alpha : [0,1] \to [0,1]$ satisfying $\alpha_0=0$ and $\alpha_1=1$, the schedule $\ba = \ba^{N,K}$ is defined as 
\begin{align}
\label{eq:link_a_alpha}
a^{N,K}_k &= \lceil N \alpha_{k/K} \rceil
&k=0,\dots,K.
\end{align}
By construction $\ba^{N,K}$ is increasing with $a^{N,K}_0= 0$ and $a^{N,K}_K = N$, and the curve $(\alpha_t)_{t \in [0,1]}$ is a suitable limit of a rescaled version of the schedule $\ba^{N,K}$ as $N, K \to + \infty$.

\begin{remark}
Various authors have proposed to define the non-constant schedule $\ba$ through a continuous function, see e.g.\ \citet{shi2024simplified} and references therein.
This is usually done in the context of MDMs defined through continuous-time Markov chains where the sizes $s_k$ are usually random, for example $s_k\sim \mathrm{Binomial}(N-a_{k-1},p_k)$ with $p_k$ depending on the specific discretization mechanism used and on a continuous functions that specifies the schedule; see e.g.\ the function $\alpha$ defined in equation (1) of \citet{shi2024simplified}.
\end{remark}

Under suitable assumptions, the problem~\eqref{eq:optimal_schedule} with the ansatz~\eqref{eq:link_a_alpha} converges to a problem of calculus of variations in the variable $(\alpha_t)_{t \in [0,1]}$. 
Specifically, assume that $(\pd^N)_{N \geq 1}$ is a sequence of probability distributions, with $\pd^N$ a probability distribution on $\sX^N$. We write $f^N : \{ 0, \ldots, N-1 \} \to \R$ for the information profile of $\pd^N$ and $g^N : [0,1-\frac{1}{N}) \to \R_+$ for the rescaled version of $\Delta f^N$: 
specifically $g^N$ is the piecewise constant function: 
\begin{align}\label{eq:def_gN}
g^N \left( u \right) = \frac{N}{D(\pd^N)} \Delta f^N(i) \quad \text{for } u \in \left[ \frac{i-1}{N}, \frac{i}{N} \right) & & i =1, \ldots, N-1. 
\end{align}
By Lemma~\ref{lm:prop_f} we see that $g^N \geq 0$ and $\int_0^{1-1/N} g^N(u) \rmd u = 1$, which explains the normalization we choose for $g^N$.
We will assume that $g^N$ converges, as $N \to + \infty$, to a continuous function, the scaling limit of the derivative of the information profile. Specifically we make the following assumptions on $g^N$ and on the curve $(\alpha_t)_{t \in [0,1]}$ used in~\eqref{eq:link_a_alpha}.

\begin{assumption}\label{ass:reg}
$(\alpha_t)_{t \in [0,1]}$ is of class $C^1$ and 
$g^N$ converges uniformly to a continuous function $g : [0,1] \to 
\R_+$ as $N\to\infty$.
\end{assumption}

\subsection{The case of a diverging number of unmasked variables}
We first assume $N,K \to + \infty$ and $N/K \to + \infty$. 
In this case the average number of unmasked variables $s_k$ diverges: though arguably less relevant in practice, this limit is the simplest and the limiting problem can be solved in closed form. The proof of the following result, quite technical, can be found in the appendix.

\begin{theorem}\label{thm:scal_lim_diverge}
Under Assumption \ref{ass:reg}, if $K, N \to + \infty$ with $N/K\to + \infty$ then 
\begin{equation*}
A^N(\ba^{N,K}) = \frac{D(\pd^N)}{2} \frac{N}{K} \left( \int_0^1 g(\alpha_t) \dot{\alpha}_t^2 \, \rmd t + \smallo(1) \right)\,.
\end{equation*}
\end{theorem}

The above shows that $A^N(\ba^{N,K})$ is asymptotically equivalent to 
$\frac{D(\pd^N)}{2} \frac{N}{K}c(\alpha)$ with $c(\alpha)=\int_0^1 g(\alpha_t) \dot{\alpha}_t^2 \, \rmd t$. In other words, the schedule's shape $\alpha$ influences the limiting value of $E_\text{fact}$ through the multiplicative factor $c(\alpha)$. 
Thus we can look for the schedule which minimizes $c(\alpha)$, which gives a very classical problem of calculus of variations, a geodesic problem in a non-uniform environment, described by the metric tensor $g$. We report the solution of this problem, with the proof in appendix for completeness.

\begin{proposition}
\label{prop:calculus_var}
If $g$ is continuous and strictly positive, the solution to the problem of calculus of variations
\begin{align*}
\min_{\alpha : [0,1] \to [0,1]} & \int_0^1 g(\alpha_t) \dot{\alpha}_t^2 \, \rmd t,   & \text{such that} \quad  \alpha_0 = 0, \; \alpha_1 = 1, 
\end{align*}
is $\alpha_t = G^{-1}( t G(1))$, with $G(y) = \int_0^y \sqrt{g(u)} \rmd u$ an antiderivative of $\sqrt{g}$. The optimal value is 
\begin{equation*}
\int_0^1 g(\alpha_t) \dot{\alpha}_t^2 \, \rmd t =  \left( \int_0^1 \sqrt{g(u)} \, \rmd u \right)^2.   
\end{equation*}
\end{proposition}

As $\int_0^{1-1/N} g^N(u) \rmd u = u$, passing to the limit we have $\int_0^1 g(u) \rmd u = 1$.
Thus the ratio between the optimal constant $c(\alpha^\text{opt})$ and the one of the uniform schedule $c(\alpha^\text{unif})$ corresponding to $\alpha^\text{unif}_t = t$ is given by 
\begin{equation*}
\frac{c(\alpha^\text{opt})}{c(\alpha^\text{unif})} = \left( \int_0^1 \sqrt{g(u)} \, \rmd u \right)^2 = \frac{ \left( \int_0^1 \sqrt{g(u)} \rmd u  \right)^2}{\int_0^1 g(u) \, \rmd u}.
\end{equation*}
This quantity is always smaller than $1$ thanks to Jensen's inequality, and more interestingly it becomes smaller if there is a bigger gap in Jensen's inequality.  This formalizes the intuition that non-constant schedules are more beneficial the further the information profile is from linear.

The optimal continuous schedule is thus given by $\alpha_t = G^{-1}( t G(1))$. Interestingly, we can check easily that $\alpha$ is convex (resp.\ concave) if $g$ is non-increasing (resp.\ non-decreasing), which is consistent with Proposition~\ref{prop:f_a_convex_concave}.
In the (very likely) case where $g$ is not available in closed from, it makes sense to define the schedule in a data driven way. To do so, first note that $g \approx  \frac{N}{D(\pd^N)} \Delta f^N$ and thus $\sqrt{\frac{D(\pd^N)}{N}} G \left( \frac{n}{N} \right) \approx \sum_{i=0}^{n-1} \sqrt{\Delta f^N(i)}$, leading to the schedule
\begin{equation}
\label{eq:data_driven_schedule}
a^{N,K}_k = \min \left\{ n \ : \ \sum_{i=0}^{n-1} \sqrt{\Delta f^N(i)} \geq \frac{k}{K}  \sum_{i=0}^{N-1} \sqrt{\Delta f^N(i)} \right\}, 
\end{equation}
which we expect to be asymptotically optimal by Theorem \ref{thm:scal_lim_diverge} and Proposition \ref{prop:calculus_var}.
Equation \eqref{eq:data_driven_schedule} can then be used in conjunction with empirical estimates of the information profile $f$ to define data-driven optimal schedule sizes.
For instance, estimates of $f$ can be obtained by approximating $f(i)=
\E_{\pd(\bx),\sigma\sim \Unif}\left[
\log \pi(x_{\sigma_{i+1}}|\bx_{\sigma_{\leq i}})
\right]$ as
\begin{equation}\label{eq:est_f_i}
f(i)
\approx
\frac{1}{N-i}\sum_{j\notin \bz}
\sum_{\ell\in\sX}\pth(x_j=\ell;\bx_{\bz})\log\pth(x_j=\ell;\bx_{\bz})\,,    
\end{equation}
with $\bz$ sampled uniformly at random from the subsets of $\{ 1, \ldots,N \}$ of size $i$, and 
$\bx$ being a sample from $\pi$ (obtained by picking it uniformly from the available training dataset). 
The expectation of the right-hand side of \eqref{eq:est_f_i} is $
\E_{\pd(\bx),\sigma\sim \Unif}\left[
\log \pth(x_{\sigma_{i+1}};\bx_{\sigma_{\leq i}})
\right]$, which coincides with $f(i)$ modulo the error in the approximation $\pth\approx \pi$.
The estimator in \eqref{eq:est_f_i} can then be combined with, e.g.,  kernel smoothing or other variance reduction techniques to obtain an estimate of the function $i\mapsto f(i)$ that can be used to approximate the asymptotically optimal schedule defined in \eqref{eq:data_driven_schedule}.
We leave more discussion and exploration of low-variance estimators of $f$, $\Delta f$ and $a^{N,K}$ to future work.

\begin{remark}[$\Gamma$-convergence]
Theorems \ref{thm:scal_lim_diverge}, and \ref{thm:scal_lim_bounded} below, only analyze the pointwise limit of $A^N(\ba^{N,K})$ as $N,K \to + \infty$ under Assumption~\ref{ass:reg}. A proper mathematical analysis would require the $\Gamma$-convergence of $A^N(\cdot)$ in order to guarantee that convergence of the optimizers and minimal value of $A^N(\cdot)$ to the problems of calculus of variations of Proposition~\ref{prop:calculus_var}. Given that the pointwise limit is already quite technical to prove and allows us to draw interesting conclusions, we do not pursue this avenue here. 
\end{remark}

Other works, such as \citep{zhang2025cosine}, also propose schedules minimizing a problem of calculus of variations having a structure similar to ours. However, we emphasize that the objective we optimize is directly related to the approximation error of the algorithm (through the factorization error), and the optimal schedule we obtain depends on the target distribution (through the information profile).

\subsection{The case of a bounded number of unmasked variables}

We now turn to the case where the typical size $N/K$ does not diverge but rather converges to a finite limit $\bar s \in[1,\infty)$. We define $h_{\bar{s}} : [0, + \infty) \to [0, + \infty)$ as the continuous function such that $h_{\bar{s}}(u) = u^2$ if $u =n/\bar{s}$ for $n \in \mathbb{N}$, and which is piecewise affine in between: in formula, 
\begin{equation*}
h_{\bar{s}}(u) = \frac{1}{\bar{s}^2} (1 - \{ \bar{s} u \}) \lfloor \bar{s} u \rfloor^2 + \frac{1}{\bar{s}^2} \{ \bar{s} u  \} \lceil \bar{s} u \rceil^2,
\end{equation*} 
where $\{ u \} = u - \lfloor u \rfloor \in [0,1)$ denotes the fractional part of $u$. The proof of the following can be found in the appendix.

\begin{theorem}\label{thm:scal_lim_bounded}
Under Assumption \ref{ass:reg}, and assuming that $\dot{\alpha}$ has a finite number of minimum and maximum points, if $K, N \to + \infty$ with $N/K\to \bar{s}\in[1,\infty)$ there holds
\begin{equation*}
A^N(\ba^{N,K}) = \frac{D(\pd^N)}{2}  \left( \bar{s} \int_0^1 g(\alpha_t) h_{\bar{s}} (\dot{\alpha}_t) \, \rmd t - 1 + \smallo(1) \right). 
\end{equation*}
\end{theorem}

\begin{remark}[Quantization effect]
The analysis of the case $K, N \to + \infty$ but $N/K \to \bar{s}$ is more delicate, the typical size $\bar{s}$ is still present in the expression of the limit. The explanation for it is the following: from~\eqref{eq:link_a_alpha} we should have $s_{k+1} = a_{k+1} - a_k \approx \frac{N}{K} \dot{\alpha}_{k/K}$, however we also know that $a_{k+1} - a_k$ should be an integer. Thus the difference between $a_{k+1} - a_k$ and $\frac{N}{K} \dot{\alpha}_{k/K} \approx \bar{s} \dot{\alpha}_{k/K}$ does not vanish in the limit, there is a quantization effect still present due to the transformation of the continuous variable $N \alpha_{k/K}$ into a integer-valued one in \eqref{eq:link_a_alpha}.
This quantization effect explains why there is the term $h_{\bar{s}}(\dot{\alpha}_t)$ in the limiting integral, which should be thought as an approximation of the square function, instead of $\dot{\alpha}_t^2$ as in Theorem~\ref{thm:scal_lim_diverge}.   
As $\bar{s} \to + \infty$, we have $h_{\bar{s}}(u) \to u^2 $ for all $u$, so that the integral term $\int_0^1 g(\alpha_t) h_{\bar{s}} (\dot{\alpha}_t) \, \rmd t$ converges to $\int_0^1 g(\alpha_t) \dot{\alpha}_t^2 \, \rmd t$ featured in the previous limit. 
\end{remark}

Contrary to the previous case, minimizing the integral $\int_0^1 g(\alpha_t) h_{\bar{s}} (\dot{\alpha}_t) \, \rmd t$ in order to look for an asymptotically optimal schedule looks more challenging, in particular because the function $h_{\bar{s}}$ is not of class $C^1$. However as $\bar s$ increases we expect the solution to the previous case to be close to optimal. 
We can make the latter claim quantitative as follows: by the convexity of the square function we can check $u^2 \leq h_{\bar{s}}(u) \leq u^2 + \frac{1}{4 \bar{s}^2}$, thus
\begin{equation*}
A^N(\ba^{N,K}) \leq \frac{D(\pd^N)}{2}  \left( \bar{s} \int_0^1 g(\alpha_t) \dot{\alpha}_t^2 \, \rmd t + \frac{\sup g}{4 \bar{s}}  - 1 + \smallo(1) \right).
\end{equation*}
This yields an asymptotic bound if we use the schedule $\alpha_t = G^{-1}(t G(1))$ given in Proposition~\ref{prop:calculus_var} which is better than the uniform schedule, if $\bar s$ is large enough.

\section{Extensions and future work}\label{sec:disc}
It would be interesting to extend our results in various directions. For example, the sampling algorithms we study in this work (i.e.\ those falling into the framework of Algorithm \ref{alg:gen}) are not allowed to remove or modify any coordinate $x_i$ after having generated it. Recently, various authors considered improving and generalizing MDMs by adding so-called corrector or remasking steps, where coordinates can be removed or resampled after being generated, see e.g.\ \citep{liu2024think,zhao2024informed,wang2025remasking}. It would be valuable to extend our theoretical results to these settings to explore and quantify potential benefits of them.
Similarly, it would be interesting to analyse  methods that employ adaptive planners where $z_k$ depends on $\bx_{\bz_{<k}}$, see e.g.\ \citep{ben2025accelerated,luxembourg2025plan,kim2025train,peng2025path}, in order to theoretically quantify the gains they can obtain relative to random-order strategies that satisfy \eqref{eq:random_order}.
Finally, while we developed our theory assuming $\sX$ to be a discrete and finite space (which is the typical setting in applications of MDMs) we expect all our results to directly extend to general state spaces $\sX$ with the only modification that $\log|\sX|=\infty$ in such case, but $D(\pd)$ is still a finite quantity.

\section*{Acknowledgments}
The authors thank Michalis K.\ Titsias for useful conversations.

\bibliographystyle{abbrvnat}

\bibliography{bibliography}

\begin{thebibliography}{20}
\providecommand{\natexlab}[1]{#1}
\providecommand{\url}[1]{\texttt{#1}}
\expandafter\ifx\csname urlstyle\endcsname\relax
  \providecommand{\doi}[1]{doi: #1}\else
  \providecommand{\doi}{doi: \begingroup \urlstyle{rm}\Url}\fi

\bibitem[Austin et~al.(2021)Austin, Johnson, Ho, Tarlow, and Van Den~Berg]{austin2021structured}
J.~Austin, D.~D. Johnson, J.~Ho, D.~Tarlow, and R.~Van Den~Berg.
\newblock Structured denoising diffusion models in discrete state-spaces.
\newblock \emph{Advances in neural information processing systems}, 34:\penalty0 17981--17993, 2021.

\bibitem[Austin(2020)]{Austin2020Multivariate}
T.~Austin.
\newblock Multi-variate correlation and mixtures of product measures.
\newblock \emph{Kybernetika}, 56\penalty0 (3):\penalty0 459--499, 2020.

\bibitem[Bauer et~al.(2008)Bauer, Schuster, and Sayood]{bauer2008average}
M.~Bauer, S.~M. Schuster, and K.~Sayood.
\newblock The average mutual information profile as a genomic signature.
\newblock \emph{BMC bioinformatics}, 9\penalty0 (1):\penalty0 48, 2008.

\bibitem[Ben-Hamu et~al.(2025)Ben-Hamu, Gat, Severo, Nolte, and Karrer]{ben2025accelerated}
H.~Ben-Hamu, I.~Gat, D.~Severo, N.~Nolte, and B.~Karrer.
\newblock {Accelerated Sampling from Masked Diffusion Models via Entropy Bounded Unmasking}.
\newblock \emph{arXiv preprint arXiv:2505.24857}, 2025.

\bibitem[Billingsley(1995)]{Billingsley}
P.~Billingsley.
\newblock \emph{Probability and Measure}.
\newblock John Wiley \& Sons, 3rd edition, 1995.

\bibitem[Campbell et~al.(2022)Campbell, Benton, De~Bortoli, Rainforth, Deligiannidis, and Doucet]{campbell2022continuous}
A.~Campbell, J.~Benton, V.~De~Bortoli, T.~Rainforth, G.~Deligiannidis, and A.~Doucet.
\newblock A continuous time framework for discrete denoising models.
\newblock \emph{Advances in Neural Information Processing Systems}, 35:\penalty0 28266--28279, 2022.

\bibitem[Chen et~al.(2025)Chen, Cong, and Li]{chen2025optimal}
S.~Chen, K.~Cong, and J.~Li.
\newblock {Optimal Inference Schedules for Masked Diffusion Models}.
\newblock \emph{arXiv preprint arXiv:2511.04647}, 2025.

\bibitem[Cover and Thomas(2006)]{CoverThomas2006}
T.~M. Cover and J.~A. Thomas.
\newblock \emph{{Elements of Information Theory}}.
\newblock Wiley-Interscience, 2nd edition, 2006.

\bibitem[Kim et~al.(2025)Kim, Shah, Kontonis, Kakade, and Chen]{kim2025train}
J.~Kim, K.~Shah, V.~Kontonis, S.~Kakade, and S.~Chen.
\newblock Train for the worst, plan for the best: Understanding token ordering in masked diffusions.
\newblock \emph{arXiv preprint arXiv:2502.06768}, 2025.

\bibitem[Li and Cai(2025)]{li2025convergence}
G.~Li and C.~Cai.
\newblock A convergence theory for diffusion language models: An information-theoretic perspective.
\newblock \emph{arXiv preprint arXiv:2505.21400}, 2025.

\bibitem[Liu et~al.(2024)Liu, Nam, Campbell, St{\"a}rk, Xu, Jaakkola, and G{\'o}mez-Bombarelli]{liu2024think}
S.~Liu, J.~Nam, A.~Campbell, H.~St{\"a}rk, Y.~Xu, T.~Jaakkola, and R.~G{\'o}mez-Bombarelli.
\newblock Think while you generate: Discrete diffusion with planned denoising.
\newblock \emph{arXiv preprint arXiv:2410.06264}, 2024.

\bibitem[Luxembourg et~al.(2025)Luxembourg, Permuter, and Nachmani]{luxembourg2025plan}
O.~Luxembourg, H.~Permuter, and E.~Nachmani.
\newblock Plan for speed--dilated scheduling for masked diffusion language models.
\newblock \emph{arXiv preprint arXiv:2506.19037}, 2025.

\bibitem[Park et~al.(2024)Park, Lai, Hayakawa, Takida, and Mitsufuji]{park2024optimizing}
Y.-H. Park, C.-H. Lai, S.~Hayakawa, Y.~Takida, and Y.~Mitsufuji.
\newblock Jump your steps: Optimizing sampling schedule of discrete diffusion models.
\newblock \emph{arXiv preprint arXiv:2410.07761}, 2024.

\bibitem[Peng et~al.(2025)Peng, Bezemek, Patel, Rector-Brooks, Yao, Bose, Tong, and Chatterjee]{peng2025path}
F.~Z. Peng, Z.~Bezemek, S.~Patel, J.~Rector-Brooks, S.~Yao, A.~J. Bose, A.~Tong, and P.~Chatterjee.
\newblock Path planning for masked diffusion model sampling.
\newblock \emph{arXiv preprint arXiv:2502.03540}, 2025.

\bibitem[Sahoo et~al.(2024)Sahoo, Arriola, Schiff, Gokaslan, Marroquin, Chiu, Rush, and Kuleshov]{sahoo2024simple}
S.~Sahoo, M.~Arriola, Y.~Schiff, A.~Gokaslan, E.~Marroquin, J.~Chiu, A.~Rush, and V.~Kuleshov.
\newblock Simple and effective masked diffusion language models.
\newblock \emph{Advances in Neural Information Processing Systems}, 37:\penalty0 130136--130184, 2024.

\bibitem[Shi et~al.(2024)Shi, Han, Wang, Doucet, and Titsias]{shi2024simplified}
J.~Shi, K.~Han, Z.~Wang, A.~Doucet, and M.~Titsias.
\newblock Simplified and generalized masked diffusion for discrete data.
\newblock \emph{Advances in neural information processing systems}, 37:\penalty0 103131--103167, 2024.

\bibitem[Uria et~al.(2014)Uria, Murray, and Larochelle]{uria2014deep}
B.~Uria, I.~Murray, and H.~Larochelle.
\newblock A deep and tractable density estimator.
\newblock In \emph{International Conference on Machine Learning}, pages 467--475. PMLR, 2014.

\bibitem[Wang et~al.(2025)Wang, Schiff, Sahoo, and Kuleshov]{wang2025remasking}
G.~Wang, Y.~Schiff, S.~S. Sahoo, and V.~Kuleshov.
\newblock Remasking discrete diffusion models with inference-time scaling.
\newblock \emph{arXiv preprint arXiv:2503.00307}, 2025.

\bibitem[Zhang and Syed(2025)]{zhang2025cosine}
L.~Zhang and S.~Syed.
\newblock {The cosine schedule is Fisher-Rao-optimal for masked discrete diffusion models}.
\newblock \emph{arXiv preprint arXiv:2508.04884}, 2025.

\bibitem[Zhao et~al.(2024)Zhao, Shi, Chen, Druckmann, Mackey, and Linderman]{zhao2024informed}
Y.~Zhao, J.~Shi, F.~Chen, S.~Druckmann, L.~Mackey, and S.~Linderman.
\newblock Informed correctors for discrete diffusion models.
\newblock \emph{arXiv preprint arXiv:2407.21243}, 2024.

\end{thebibliography}

\appendix

\section{Toy example: the Gaussian case}\label{sec:gauss_info_profile}

In this appendix we consider the case of Gaussian multivariate distributions for which the information profile can be computed explicitly. 
It provides a concrete example where the convexity or concavity of the information profile depends explicitly on the model parameters.
It also emphasizes that the notions we discuss should also apply when $\pd$ is a continuous distribution.

We take $\sX = \R$ and $\pd$ to be a centered Gaussian over $\R^N$ with covariance matrix
\begin{equation*}
\Sigma = 
\begin{pmatrix}
1 & \rho & \ldots & \rho \\
\rho & 1 & \ddots & \vdots \\
\vdots & \ddots & \ddots & \rho \\
\rho & \ldots & \rho & 1 
\end{pmatrix},
\end{equation*}
with $\rho \in [-\frac{1}{N-1}, 1]$ the pairwise correlation. In the extreme case $\rho =1$ then $x_1 = \ldots = x_N$ a.s., while in the extreme case $\rho= - \frac{1}{N-1}$ we rather have $x_1+ \ldots + x_N = 0$ a.s..
The intermediate case $\rho = 0$ corresponds to independent components.

Elementary computation yields that $\pd(x_{\sigma_{i+1}} | x_{\sigma_{\leq i}} )$ is a Gaussian distribution of variance $\frac{(1-\rho)(1+i \rho)}{1 + (i-1)\rho}$.
Thus we find the information profile: 
\begin{equation*}
f(i) = - \frac{1}{2} \left[ \log(2 \pd e) + \log(1-\rho) + \log \frac{1+i\rho}{1 + (i-1)\rho}   \right].
\end{equation*}
It is convex for $\rho \leq 0$ and concave for $\rho \geq 0$, as well as strictly increasing. From Proposition~\ref{prop:f_a_convex_concave}, we deduce that the optimal schedule should select $(s_k)_k$ decreasing for $\rho < 0$ and increasing for $\rho > 0$. It is compatible with the following intuition: if $\rho > 0$ increases, we are leaning towards $x_1 = \ldots = x_N$ a.s. and $x_2, \ldots, x_N$ become more deterministic and independent conditionally to $x_1$. On the other hand, if $\rho < 0$ gets closer to $-\frac{1}{N-1}$, we lean towards the extreme case $x_1+ \ldots + x_N = 0$ a.s..
In this case, it is at the end of the sampling that $s_k$ should be small in order to sample accurately the last components and enforce $x_1+ \ldots + x_N = 0$.

We also find
\begin{equation*}
D(\pd) = \frac{1}{2} \log \left( \frac{1 + (N-2) \rho}{1+(N-2) \rho - (N-1) \rho^2} \right).
\end{equation*}
To study a scaling limit, we fix $\xi \in (-1, + \infty)$ and consider $\rho^N = \frac{\xi}{N}$. With $\pd^N$ the Gaussian measure above with $\rho = \rho^N$, we obtain 
$D(\pd^N)\sim \frac{\xi^2}{2N(1+\xi)}$ as $N\to\infty$.
If we look at $g^N$ the derivative of the renormalized information profile as in Section~\ref{sec:scaling_limit_setting}, we obtain that $g^N \to g$ uniformly with $g(u) = \frac{1+\xi}{(1+ \xi u)^2}$. 
In particular, with Theorem~\ref{thm:scal_lim_diverge} and Proposition~\ref{prop:calculus_var} suggest to use in the limit $K,N \to + \infty$ with $N/K \to + \infty$ the exponential schedule:
\begin{equation*}
\alpha_t = \frac{(1+\xi)^t - 1}{\xi} \, , \qquad \text{yielding} \qquad E_\text{fact} \sim \frac{\ln(1+\xi)^2}{4K} \, .   
\end{equation*}

\section{Additional proofs}

\subsection{Auxiliary results}

We collect here the proof of several auxiliary results, not necessarily technical, but whose proof would have broken the flow of the main manuscript.

We start with Lemma \ref{lemma:arbitrary_schedule}.  Before proving the lemma we note that, given $\bx$, $\nuth(\bz;\bx)$ is a probability mass function over ordered partitions of $\{ 1, \ldots, N \}$, in the sense $\sum_{\bz} \nuth(\bz;\bx) = 1$, however it does \emph{not} coincide with the conditional distribution $\pb(\bz|\bx)$. Similarly given a schedule $\bz$ of unmasking, $\pth(\bx;\bz)$ is a probability mass function over $\sX^N$ which does not coincide with the conditional distribution $\pb(\bz|\bx)$.  

\begin{proof}[Proof of Lemma \ref{lemma:arbitrary_schedule}]
Under the assumption and with the notations above, for any partition $\bz$, 
$$
\pth(\bz;\bx)
=
\prod_{k=1}^K \pth(\bx_{z_k};\bx_{\bz_{<k}})
=
\prod_{k=1}^K \pd(\bx_{z_k}|\bx_{\bz_{<k}})
=
\pd(\bx)
$$
does not depend on $\bz$ and thus 
$$
\pb(\bx)
=
\sum_\bz \pb(\bx,\bz)
=
\sum_\bz \pd(\bx)\nuth(\bz;\bx)
=
\pd(\bx)\sum_\bz \nuth(\bz;\bx)=\pd(\bx)\,. \qedhere
$$
\end{proof}

\begin{proof}[Proof of Lemma \ref{lm:discrete_der}]
We start from $(a_{k+1} - a_k) f(a_k) = \sum_{i=a_k}^{a_{k+1}-1} f(a_k)$. Thus grouping the terms in the definition of $A$ and using the definition of $\Delta f$
\begin{equation*}
A(\ba)= \sum_{k=0}^{K-1} \sum_{i=a_k}^{a_{k+1}-1} (f(i) - f(a_k)) = \sum_{k=0}^{K-1} \sum_{i=a_k}^{a_{k+1}-1} \sum_{j=a_k +1}^i \Delta f(j).
\end{equation*}
For a fixed $k$ we exchange the order of the inner double summation:
\begin{equation*}
\sum_{i=a_k}^{a_{k+1}-1} \sum_{j=a_k +1}^i \Delta f(j) = \sum_{j=a_k +1}^{a_{k+1}-1} \sum_{i=j}^{a_{k+1}-1} \Delta f(i) = \sum_{j=a_k +1}^{a_{k+1}-1} (a_{k+1} - j) \Delta f(j).
\end{equation*}
If $j \in \{ a_k +1, a_{k+1}-1 \}$ then $a_{k+1} = r_\ba(j)$. Moreover the sum in $j$ could be extended to $j = a_{k+1}$ as in this case $r_\ba(j) - j = 0$. Thus the latter sum coincide with $\sum_{j=a_k +1}^{a_{k+1}} (r_\ba(j) - j) \Delta f(j)$. Summing over $k$ gives the expression in \eqref{eq:discr_der_bound}.
\end{proof}

\begin{proof}[Proof of Proposition \ref{prop:E_fac_geom}]
We use the notations of Lemma~\ref{lm:discrete_der}.
By definition of $r_{\ba}(i)$ and the memoryless property of the Geometric distribution we have $(1+r_{\ba}(i)-i)\sim \Geom (p;N-i+1)$. 
Since the expectation of a $\Geom (p;m)$ distribution is $(1-(1-p)^m)/p$, it follows that 
$$\E[r_{\ba}(i)-i]=\frac{1}{p}\left(1-(1-p)^{N-i+1}\right)-1=\frac{1-p}{p}\left(1-(1-p)^{N-i}\right)\,.$$
Thus, by \eqref{eq:discr_der_bound}, and using $\sum \Delta f(i) = f(N-1) - f(0) = D(\pd)$,
\begin{align*}
E_\text{fact}
=
\sum_{i=1}^{N-1} 
\E[(r_{\ba}(i)-i)]\Delta f(i)
&=
\frac{1-p}{p}
\sum_{i=1}^{N-1}(1-(1-p)^{N-i})\Delta f(i)\\
&
= \frac{1-p}{p}D(\pd) - \frac{1-p}{p} \sum_{i=1}^{N-1} (1-p)^{N-i} \Delta f(i).
\end{align*}
The upper bound follows directly from $\Delta f \geq 0$. For the lower bound, we bound $\Delta f(i)$ by its maximum and use $\sum_{i=1}^{N-1} (1-p)^{N-i} \leq \sum_{i=1}^{\infty} (1-p)^{N-i} =  (1-p)/p$.
\end{proof}

\begin{proof}[Proof of Proposition \ref{prop:f_a_convex_concave}]
We start by deriving some optimality conditions for an optimal schedule. That is, we fix $\ba$ a solution of~\eqref{eq:optimal_schedule}. For a given $k = 1, \ldots, K-1$, writing $\ba'_{\pm} = (a_0, a_1, \ldots, a_{k-1}, a_k \pm 1, a_{k+1}, \ldots, a_K)$ and expanding $A(\ba'_\pm) \geq A(\ba)$, we obtain the two necessary conditions:
\begin{align}
\label{eq:optimality_1}
(a_{k+1} - a_k) \Delta f(a_k +1)
+ f(a_{k-1}) - f(a_{k}+1) & \leq 0, \\
\label{eq:optimality_2}
-(a_{k+1} - a_k) \Delta f(a_k)
+ f(a_{k}-1) - f(a_{k-1})  & \leq 0.
\end{align}
Reordering these equations we find~\eqref{eq:sequential_sol}. 

Next we assume that $f$ is strictly convex and strictly increasing. By strict convexity we have $f(a_{k}+1) - f(a_{k-1}) < (a_k + 1 - a_{k-1}) \Delta f(a_k +1)$. Plugging this in~\eqref{eq:optimality_1} we obtain 
\begin{equation*}
(a_{k+1} - a_k) \Delta f(a_k +1) < (a_k + 1 - a_{k-1}) \Delta f(a_k +1).
\end{equation*}
Dividing by $\Delta f(a_k +1) > 0$, we obtain $a_{k+1} - a_k < a_k  - a_{k-1} + 1$. 
Since the $a_k$'s are integer, we conclude $a_{k+1} - a_k \leq a_k  - a_{k-1}$. As this is valid for any $k$, the sequence $(s_k)_k$ is non-increasing.  

Then, we assume that $f$ is strictly concave and strictly increasing. We use~\eqref{eq:optimality_2} together with $f(a_{k}-1) - f(a_{k-1}) > (a_k - 1 - a_{k-1}) (f(a_{k}) - f(a_{k} - 1))$ by strict concavity, so that 
\begin{equation*}
(a_k - 1 - a_{k-1}) \Delta f(a_k)  < (a_{k+1} - a_k)\Delta f(a_k). 
\end{equation*}
We divide by $\Delta f(a_k) > 0$ and obtain $a_k  - a_{k-1} - 1 < a_{k+1} - a_k  $, so that $a_k  - a_{k-1}  \leq a_{k+1} - a_k  $ as they are integers. The conclusion follows as $k$ is arbitrary. 
\end{proof}

\subsection{A classical problem of calculus of variations}

We report here the proof of Proposition \ref{prop:calculus_var} for completeness. Though the solution can be find by solving the Euler-Lagrange equation, in this case we rely only on Jensen's inequality. 

\begin{proof}[Proof of Proposition \ref{prop:calculus_var}]
As $g$ is strictly positive and continuous, the function $G$ is a $C^1$ diffeomorphism. Thus if $(\alpha_t)_{t}$ is any competitor, we can consider $\beta_t = G(\alpha_t)$ which is also a curve of class $C^1$. As $\dot{\beta}_t = \sqrt{g(\alpha_t)} \dot{\alpha_t}$, we find with Jensen's inequality
\begin{equation*}
\int_0^1 g(\alpha_t) \dot{\alpha}_t^2 \, \rmd t = \int_0^1 \dot{\beta}_t^2 \, \rmd t \geq \left( \int_0^1 \dot{\beta}_t \, \rmd t \right)^2 = (G(\alpha_1) - G(\alpha_0))^2 = G(1)^2.  
\end{equation*}
Moreover there is equality if and only if the function $(\dot{\beta}_t)_{t}$ is constant, which can only happen if $\beta_t = t G(1)$ for all $t$, leading to $\alpha_t = G^{-1}(t G(1))$.
\end{proof}

\subsection{Proof of Theorem~\ref{thm:scal_lim_diverge} and Theorem~\ref{thm:scal_lim_bounded}: the scaling limits}

We present here the proofs of Theorem~\ref{thm:scal_lim_diverge} and Theorem~\ref{thm:scal_lim_bounded}. They are longer and more technical than the others proofs, and we put them in an appendix to avoid breaking the flow of the presentation.

\underline{A preliminary observation.}
We collect the following bound which follows from the definition of $\ba^{N,K}$ in~\eqref{eq:link_a_alpha}: for any $k$, 
\begin{equation}
\label{eq:bound_s_k_apriori}
0 \leq a^{N,K}_{k+1} - a^{N,K}_{k} \leq 1 + \frac{N}{K} \sup_{t \in [0,1]} \dot{\alpha}_t. 
\end{equation}

\underline{1st step: some algebraic manipulations.}
In order to have a slight algebraic simplification later, we will rather look at $\tilde{A}^{N} = A^N + \frac{D(\pd^N)}{2}$. As $D(\pd^N) = \sum_{i=1}^{N-1} \Delta f^N(i)$, calling $r^{N,K}(i)=\inf\{a^{N,K}_k\,:\,a^{N,K}_k\geq i\}$ we have with Lemma~\ref{lm:discrete_der}:
\begin{equation*}
\tilde{A}^N(\ba^{N,K} ) = A^N(\ba^{N,K} ) + \frac{D(\pd^N)}{2} = \sum_{i=1}^{N-1} \Delta f^N(i) \left(r^{N,K}(i)-i + \frac{1}{2} \right). 
\end{equation*}
Moreover, we group the sum by the value of $r^{N,K}$, ending up with the expression
\begin{equation}
\label{eq:expr_tildeA}
\tilde{A}^N(\ba^{N,K} ) = \sum_{k=0}^{K-1} \sum_{i=a^{N,K}_k+1}^{a^{N,K}_{k+1}} \Delta f^N(i) \left(a^{N,K}_{k+1}-i + \frac{1}{2} \right). 
\end{equation}

\underline{2nd step: transforming the objective into (almost) a Riemann sum.}
Next we claim that we have 
\begin{equation}
\label{eq:estimate_still_joint}
\tilde{A}^N(\ba^{N,K} ) = \frac{D(\pd^N)}{2N} \sum_{k=0}^{K-1} g(\alpha_{k/K}) \left(a^{N,K}_{k+1}-a^{N,K}_k  \right)^2 + \smallo \left( \frac{D(\pd^N) N}{K} \right). 
\end{equation}
Let's prove this claim.
Given the uniform convergence of $g^N$ to a continuous limit $g$, with 
\begin{equation*}
\varepsilon_N = \sup_{i = 1, \ldots, N-1} \left| \frac{N}{D(\pd^N)} \Delta f^N(i) - g \left( \frac{i}{N} \right) \right|,
\end{equation*}
we have that $\varepsilon_N \to 0$ as $N \to + \infty$. 

In the expression~\eqref{eq:expr_tildeA}, in the $k$-th block we will replace $\Delta f^N(i)$ by $\frac{D(\pd^n)}{N} g(\alpha_{k/K})$. Specifically if $i \in \{ a^{N,K}_k +1, \ldots , a^{N,K}_{k+1} \}$ we write with the triangle inequality
\begin{equation*}
\left| \frac{N}{D(\pd^N)} \Delta f^N(i) - g(\alpha_{k/K})  \right| \leq \left| \frac{N}{D(\pd^N)} \Delta f^N(i) -  g \left( \frac{i}{N} \right)  \right|  + \left|  g \left( \frac{i}{N} \right) - g(\alpha_{k/N}) \right|,
\end{equation*}
The first term is bounded by $\varepsilon_N \to 0$. To handle the second one, using~\eqref{eq:bound_s_k_apriori}, the definition of $a^{N,K}_k$ in~\eqref{eq:link_a_alpha} and the notation $C = \sup \dot{\alpha}_t$, we have:
\begin{equation*}
\left| \frac{i}{N} - \alpha_{k/N} \right|  \leq \left| \frac{i}{N} - \frac{a^{N,K}_k}{N}\right| + \left| \frac{a^{N,K}_k}{N} - \alpha_{k/N} \right| \leq \frac{a^{N,K}_{k+1} - a^{N,K}_k}{N} + \frac{1}{N} \leq \frac{C}{K} + \frac{2}{N}, 
\end{equation*}
Thus if $\omega$ is a modulus of continuity of $g$, we conclude putting the two pieces together that
\begin{equation*}
\left| \frac{N}{D(\pd^N)} \Delta f^N(i) - g(\alpha_{k/K})  \right| \leq \omega \left( \frac{C}{K} + \frac{2}{N} \right) + \varepsilon_N. 
\end{equation*}
Summing all these error terms in the expression~\eqref{eq:expr_tildeA}, using $a^{N,K}_{k+1}-i + \frac{1}{2} \leq a^{N,K}_{k+1}-a^{N,K}_k + \frac{1}{2}$, 
\begin{multline*}
\left| \tilde{A}^N(\ba^{N,K} ) - \frac{D(\pd^N)}{N} \sum_{k=0}^{K-1} \sum_{i=a^{N,K}_k+1}^{a^{N,K}_{k+1}} g(\alpha_{k/K}) \left(a^{N,K}_{k+1}-i + \frac{1}{2} \right) \right| \\
\leq \frac{D(\pd^N)}{N} \cdot \left\{ \sum_{k=0}^{K-1} \left(a^{N,K}_{k+1}-a^{N,K}_k + \frac{1}{2} \right)^2 \right\} \cdot \left( \omega \left( \frac{C}{K} + \frac{2}{N} \right) + \varepsilon_N \right)
\end{multline*}
In the sum in the right hand side we use that $a^{N,K}_{k+1}-a^{N,K}_k = \bigO(N/K)$ by~\eqref{eq:bound_s_k_apriori}, so that the whole right hand side is $\smallo(D(\pd^N) N / K)$. 
Since $\sum_{i=a+1}^b(b-i+1/2)=\sum_{i=1}^{b-a-1}(i+1/2)=(b-a)^2/2$ for any $0\leq a<b$ integers, we can simplify the left-hand side and obtain \eqref{eq:estimate_still_joint}.

To go further than~\eqref{eq:estimate_still_joint} we will relate $a^{N,K}_{k+1}-a^{N,K}_k$ to the derivative $\dot{\alpha}$. By the definition~\eqref{eq:link_a_alpha}, we have
\begin{equation}
\label{eq:link_Delta_a_alpha_dot}
\left| a^{N,K}_{k+1}-a^{N,K}_k -  N (\alpha_{(k+1)/K} - \alpha_{k/K}) \right| < 1 \quad \text{and} \quad a^{N,K}_{k+1}-a^{N,K}_k \in \N.
\end{equation}
At this point we need to differentiate the case $N/K\to +\infty$ and $N/K \to \bar{s}$.

\underline{3rd step (Theorem~\ref{thm:scal_lim_diverge}): convergence if $N/K\to \infty$.} 
We always have $\alpha_{(k+1)/K} - \alpha_{k/K} = \frac{\dot{\alpha}_{k/K}}{K} + \smallo(1/K)$, uniformly in $k$ because $\dot \alpha$ is bounded. Moreover, the distance between $a^{N,K}_{k+1}-a^{N,K}_k$ and $N (\alpha_{(k+1)/K} - \alpha_{k/K})$ is smaller than $1$, and is thus a $\smallo(N/K)$ as $N/K\to \infty$. We deduce
\begin{equation*}
a^{N,K}_{k+1}-a^{N,K}_k  
= \frac{N}{K} \dot{\alpha}_{k/K} + \smallo \left( \frac{N}{K} \right)
= \frac{N}{K} (\dot{\alpha}_{k/K} +\smallo(1)),
\end{equation*}
with the $\smallo(1)$ being uniform in $k$.
Plugging this in~\eqref{eq:estimate_still_joint}, we obtain  
\begin{equation*}
\tilde{A}^N(\ba^{N,K} ) = \frac{D(\pd^N)}{2N}  \cdot \frac{N^2}{K^2} \sum_{k=0}^{K-1} g(\alpha_{k/K}) (\dot{\alpha}_{k/K} + \smallo(1))^2  + \smallo\left( \frac{D(\pd^N) N}{K} \right).   
\end{equation*}
The conclusion follows: as the function $t \mapsto g(\alpha_t) \dot{\alpha}_t^2$ is continuous, we have convergence of the Riemann sum
\begin{equation*}
\frac{1}{K}\sum_{k=0}^{K-1} g(\alpha_{k/K}) (\dot{\alpha}_{k/K} + \smallo(1))^2 \to \int_0^1 g(\alpha_t) \dot{\alpha}_t^2 \, \rmd t.
\end{equation*}
Moreover $\tilde{A}^N(\ba^{N,K} ) - A^N(\ba^{N,K}) = \frac{D(\pd^N)}{2}$ which can be absorbed in the error term $\smallo(D(\pd^N) N / K)$.

\underline{3rd step bis (Theorem~\ref{thm:scal_lim_bounded}): convergence if $N/K \to \bar{s}$.}
In this case recall that~\eqref{eq:estimate_still_joint} is still valid, but now the link between $a^{N,K}_{k+1} - a^{N,K}_{k}$ and the derivative $\dot{\alpha}$ is more subtle. Given the statement of the theorem and the expression of $A^N(\ba^{N,K} ) = \tilde{A}^N(\ba^{N,K}) - \frac{D(\pd^N)}{2}$, we only need to prove that
\begin{equation}
\label{eq:to_prove}
\frac{1}{K} \sum_{k=0}^{K-1} g(\alpha_{k/K}) \left(a^{N,K}_{k+1}-a^{N,K}_k  \right)^2 \to \bar s^2 \int_0^1 g(\alpha_t) h_{\bar{s}}(\dot{\alpha}_t) \, \rmd t. 
\end{equation}

We need to analyse the distribution of the values of $a^{N,K}_{k+1} - a^{N,K}_{k}$. We start with an auxiliary Lemma which helps solidify our intuition and which will be useful later.

\begin{lemma}
\label{lm:proportion_linear}
For parameters $\beta > 0$ and $\eta \in \R$, define $b_k = \lceil \beta k + \eta \rceil$. Then the sequence $(b_{k+1} - b_k)_{k \geq 0}$ can only take the values $\lfloor \beta \rfloor$ and $\lceil \beta \rceil$. Moreover, the number of times it takes the value $\lfloor \beta \rfloor$ (resp. $\lceil \beta \rceil$) for $k = 0, \ldots, K-1$, when divided by $K$, converges to $1-\{ \beta\}$ (resp. $\{ \beta \}$) as $K\to\infty$.  
\end{lemma}

The sequence $(b_k)$ in the lemma corresponds to a particular case of the sequence $\ba^{N,K}$: when the function $\alpha$ is linear. In this case $\beta = \frac{N}{K} \dot{\alpha}_{t} \approx \bar{s} \dot{\alpha}_{t}$.

\begin{proof}
If $\beta$ is an integer the result is immediate as $b_{k+1} - b_k = \beta$ for all $k$. If $\beta$ is not an integer, it is clear that the sequence $(b_{k+1} - b_k)_{k \geq 0}$ can only take the values $\lfloor \beta \rfloor$ and $\lceil \beta \rceil$. Calling $d_{k}$ the number of times it takes the value $\lfloor \beta \rfloor$ for $k=0, \ldots, K-1$, note that
\begin{equation*}
b_K - b_0 = \sum_{k =0}^{K-1} (b_{k+1} - b_k) = d_K \lfloor \beta \rfloor + (K-d_K) \lceil \beta \rceil. 
\end{equation*}
Dividing by $K$ and taking $K \to + \infty$, the left hand side converges to $\beta$ so that 
\begin{equation*}
\beta = \lim_{K \to + \infty} \frac{d_K}{K} ( \lfloor \beta \rfloor - \lceil \beta \rceil) + \lceil \beta \rceil = \lceil \beta \rceil - \lim_{K \to + \infty} \frac{d_K}{K} 
\end{equation*}
which gives us the result $\lim_{K} \frac{d_K}{K} = \lceil \beta \rceil - \beta =  1 - \{ \beta \}$.
\end{proof}

Next we want to extend the reasoning of the lemma when the function $\alpha$ is no longer linear. For this we rely on measure theory and refer for instance to \cite{Billingsley} for the concepts and results of measure theory we will need. 

We define $\gamma^{N,K}$ a measure on $[0,1] \times \R_+$ to capture the distributions of $a^{N,K}_{k+1} - a^{N,K}_{k}$:
\begin{equation*}
\gamma^{N,K} = \frac{1}{K} \sum_{k=0}^{K-1}  \delta_{\left(\frac{k}{K}, \, a^{N,K}_{k+1} - a^{N,K}_k \right)}.
\end{equation*}
Here $\delta_{(t,p)}$ is the Dirac mass at $(t,p) \in [0,1] \times \R_+$.
By definition, if $\chi(t,p)$ is a function of two variables,
\begin{equation*}
\frac{1}{K} \sum_{k=0}^{K-1} \chi \left( \frac{k}{K}, a^{N,K}_{k+1} - a^{N,K}_{k} \right) = \iint_{[0,1] \times \R_+} \chi(t,p) \, \rmd \gamma^{N,K}(t,p).
\end{equation*}
By finding the limit of the measure $\gamma^{N,K}$ we can find the limit of the left hand side for any continuous function $\chi$, in particular prove the limit in~\eqref{eq:to_prove}. 

The measure $\gamma^{N,K}$ is a probability measure, and by the bound~\eqref{eq:bound_s_k_apriori} it is supported on compact set independent on $K,N$. Thus \cite[Thm.\ 23.9]{Billingsley}, up to extraction it converges weakly to a limit measure $\gamma$ as $N,K \to + \infty$. As the first marginal of $\gamma^{N,K}$ is $\frac{1}{K} \sum_{k=0}^{K-1} \delta_{k/K}$, we see that the first marginal of $\gamma$ is necessarily the Lebesgue measure on $[0,1]$. Moreover as $\gamma^{N,K}$ is supported on the closed set $[0,1] \times \N$, so does any of its accumulation point. We disintegrate (that is, consider the conditional distribution \cite[Thm.\ 33.3]{Billingsley}) the limit $\gamma$ with respect to its first marginal (the Lebesgue measure), obtaining a family $(\gamma_t)_{t \in [0,1]}$ of probability distributions on $\N$. We write them $\gamma_t = \sum_n d_n(t) \delta_n$, with the weights $d_n(t)$ which may depend on $t$. We obtain that the limit $\gamma$ reads 
\begin{equation*}
\gamma = \int_{0}^1 \left( \sum_{n \in \N} d_{n}(t) \delta_{(t,n)} \right) \rmd t, \qquad \text{with} \qquad \sum_{n \in \N} d_n(t) =1 \text{ for a.e. } t.
\end{equation*}
We then proceed to link $d_{n}(t)$ to $\dot{\alpha}_t$.

Take $t$ such that $\bar{s} \dot{\alpha}_t$ is not an integer. By continuity of $\dot{\alpha}_t$, we see that if $k/K$ is close enough to $t$ then $N ( \alpha_{(k+1)/K} - \alpha_{k/K}) \approx \frac{N}{K} \dot{\alpha}_t \approx \bar{s} \dot{\alpha}_t$ is not an integer. Thus given~\eqref{eq:link_Delta_a_alpha_dot} we deduce that $a^{N,K}_{k+1}-a^{N,K}_{k} \in \{ \lfloor \bar{s} \dot{\alpha}_t \rfloor, \lceil \bar{s} \dot{\alpha}_t \rceil \}$ for $N,K$ large enough. That is, the measure $\gamma^{N,K}$ is supported on $[0,1] \times \{ \lfloor \bar{s} \dot{\alpha}_t \rfloor, \lceil \bar{s} \dot{\alpha}_t \rceil \}$ in a neighbourhood of $\{ t \} \times \R_+$. Passing to the limit $N,K \to + \infty$, the same holds for $\gamma$, so that $d_n(t) = 0$ if $n \notin \{ \lfloor \bar{s} \dot{\alpha}_t \rfloor, \lceil \bar{s} \dot{\alpha}_t \rceil \}$. The unit mass condition gives $d_{\lceil \bar{s} \dot{\alpha}_t \rceil}(t) = 1 -d_{\lfloor \bar{s} \dot{\alpha}_t \rfloor}(t)$.

On the other hand take $t$ such that $\bar{s} \dot{\alpha}_t$ is an integer. As we have done the assumption that $\dot{\alpha}$ has a finite number of points of maximum and minimum, up to excluding a finite number of $t$ such that $\bar{s} \dot{\alpha}_t$ is an integer (they make a set of Lebesgue measure $0$), we have that $\dot{\alpha}_t$ is constant in a neighbourhood of $t$. We call $I$ this neighbourhood. Thus for $k/K \in I$, we have $N (\alpha_{(k+1)/K} - \alpha_{k/K}) = \frac{N}{K} \dot{\alpha}_t$. We deduce from Lemma~\ref{lm:proportion_linear} that $a^{N,K}_{k+1}-a^{N,K}_{k} = \bar{s} \dot{\alpha}_t$ for a proportion $ |\frac{N}{K} - \bar{s}| \dot{\alpha}_t$ of the indices $k$ with $k/K \in I$. Passing to the limit $N,K \to + \infty$, we have $d_n(t) = 0$ if $n \neq \bar{s} \dot{\alpha}_t$ for all $t \in I$. 

Putting these two estimates together, we obtain a refined description of the structure of $\gamma$: calling $d(t) = d_{\lfloor \bar{s} \dot{\alpha}_t \rfloor}(t)$ which is a measurable function from $[0,1]$ to $[0,1]$,
\begin{equation*}
\gamma = \int_0^1 \left(  d(t) \delta_{(t,\lfloor \bar{s} \dot{\alpha}_t \rfloor)} + (1- d(t)) \delta_{(t,\lceil \bar{s} \dot{\alpha}_t \rceil)} \right) \rmd t.
\end{equation*}
We have narrowed down the support, it remains to identify the coefficient $d(t)$. Similarly to the proof of Lemma~\ref{lm:proportion_linear}, we use the property that the sum of $a^{N,K}_{k+1} - a^{N,K}_k$ gives us back the original sequence $a^{N,K}_k$. If $t_0 \leq t_1$, by definition of $\gamma^{N,K}$, 
\begin{equation*}
\iint_{[t_0,t_1] \times \R_+} p \, \rmd \gamma^{N,K}(t,p) = \frac{1}{K} \sum_{k = \lceil t_0 K \rceil}^{\lfloor t_1 K \rfloor}  \left( a^{N,K}_{k+1} - a^{N,K}_k \right) \to \bar{s} (\alpha_{t_1} - \alpha_{t_0}).  
\end{equation*}
On other hand, as the measure $\gamma^{N,K}$ converges to $\gamma$ weakly and that the boundary of the set $[t_0,t_1] \times \R_+$ has zero measure for $\gamma$ (because the first marginal of $\gamma$ is the Lebesgue measure) we have \cite[Thm.\ 29.2]{Billingsley} 
\begin{equation*}
\iint_{[t_0,t_1] \times \R_+} p \, \rmd \gamma^{N,K}(t,p) \to \iint_{[t_0,t_1] \times \R_+} p \, \rmd \gamma(t,p) = \int_{t_0}^{t_1} \left( d(t) \lfloor \bar{s} \dot{\alpha}_t \rfloor + (1-d(t)) \lceil \bar{s} \dot{\alpha}_t \rceil \right) \rmd t.
\end{equation*}
Dividing by $t_1-t_0$ and taking the limit $t_1 \to t_0$, by the Lebesgue differentiation theorem \cite[Thm.\ 31.3]{Billingsley} we obtain for a.e. $t_0$ the identity
\begin{equation*}
\bar{s} \dot{\alpha}_{t_0} = d({t_0}) \lfloor \bar{s} \dot{\alpha}_{t_0} \rfloor + (1-d(t_0)) \lceil \bar{s} \dot{\alpha}_{t_0} \rceil,
\end{equation*}
which gives $d({t_0}) = 1- \{ \bar{s} \dot{\alpha}_{t_0} \}$. Thus we deduce finally
\begin{equation*}
\gamma = \int_{0}^1 \left( (1-\{ \bar{s} \dot{\alpha}_{t} \}) \delta_{(t,\lfloor \bar{s} \dot{\alpha}_t \rfloor)} + \{ \bar{s} \dot{\alpha}_{t} \} \delta_{(t, \lceil \bar{s} \dot{\alpha}_{t} \rceil)} \right) \rmd t.
\end{equation*}
Recall that we started the analysis by taking $\gamma$ a limit of a subsequence of $(\gamma^{N,K})$. As the expression that we find for $\gamma$ does not depend on the subsequence, we deduce that $\gamma^{N,K}$ actually converges to $\gamma$ as $N,K \to + \infty$ and $N/K \to \bar s$.

Eventually we can conclude: by weak convergence if $\phi, \psi$ are two continuous functions
\begin{align*}
\frac{1}{K} \sum_{k=1}^{K} \phi\left( \frac{k}{K} \right) \psi\left( a^{N,K}_{k+1}-a^{N,K}_k  \right) & = \iint_{[0,1] \times \R_+} \phi(t) \psi(p) \, \rmd \gamma^{N,K}(t,p) \\
& \to \iint_{[0,1] \times \R_+} \phi(t) \psi(p) \, \rmd \gamma(t,p) \\
& = \int_0^1 \phi(t) \left( (1-\{ \bar{s} \dot{\alpha}_{t} \}) \psi(\lfloor \bar{s} \dot{\alpha}_t \rfloor) +  \{ \bar{s} \dot{\alpha}_{t} \} \psi(\lceil \bar{s} \dot{\alpha}_{t} \rceil) \right) \rmd t.
\end{align*}
We apply this result for $\phi(t) = g(\alpha_t)$ and $\psi(p) = p^2$. We obtain the limit~\eqref{eq:to_prove} we need.

\end{document}